\documentclass[letterpaper]{article}
\usepackage{algorithm}
\usepackage{algorithmic}

\usepackage[numbib]{tocbibind}
\usepackage[square]{natbib}
\usepackage{natbib}
\setcitestyle{numbers,square}

\usepackage{times}
\usepackage{soul}
\usepackage{url}
\usepackage[hidelinks]{hyperref}
\usepackage[utf8]{inputenc}
\usepackage[small]{caption}
\usepackage{graphicx}
\usepackage{amsmath}
\usepackage{amsthm}
\usepackage{booktabs}
\usepackage{algorithm}
\usepackage{algorithmic}
\usepackage[switch]{lineno}

\usepackage[]{complexity}
\usepackage{amssymb}
\usepackage{thm-restate}
\usepackage{etoolbox}

\usepackage{tikz}

\usetikzlibrary{shapes.geometric}
\usetikzlibrary{positioning}
\tikzstyle{arg}=[draw,circle,fill=gray!15,inner sep=1pt,minimum size=.5cm]
\tikzstyle{targ}=[inner sep=1pt,minimum size=.5cm]

\usepackage{nicefrac}

\newtheorem{theorem}{Theorem}[section]
\newtheorem{lemma}[theorem]{Lemma}
\newtheorem{proposition}[theorem]{Proposition}
\newtheorem{corollary}[theorem]{Corollary}
\newtheorem{claim}[theorem]{Claim}

\theoremstyle{definition}
\newtheorem{definition}[theorem]{Definition}
\newtheorem{example}[theorem]{Example}

\usepackage{todonotes}
\newcommand{\TODO}[1]{\todo[inline]{#1}}

\newcommand{\commentout}[1]{}

\newcommand{\ie}{i.e., }

\newcommand{\SigmaP}[1]{\ComplexityFont{\Sigma}_{#1}^{\P}}

\newcommand{\args}{\ensuremath{A}}
\newcommand{\supporter}{\ensuremath{\mathrm{Sup}}}
\newcommand{\attacker}{\ensuremath{\mathrm{Att}}}
\renewcommand{\PF}{\mathbb{F}}
\newcommand{\PG}{\mathbb{G}}
\newcommand{\allprem}{\mathcal{P}}
\newcommand{\prem}{\ensuremath{\pi}}

\usepackage{listings}
\newcommand{\aspmodule}[1]{\Pi_{\mathit{#1}}}
\newcommand{\aspassump}{{\bf assumption}}
\newcommand{\asphead}{{\bf head}}
\newcommand{\aspbody}{{\bf body}}
\newcommand{\aspcontrary}{{\bf contrary}}
\newcommand{\aspin}{{\bf in}}

\lstdefinelanguage{asp}
{morekeywords={in,rule,contrary,defeated_by_undefeated,out,derived,derivable,head,body,derived_from_undefeated,derivable_from_undefeated,applicable_rule,usable_by_in,assumption,defeated,usable_by_undefeated,att_by_undefeated,no_undef_closed,guess,in_derives,guess_derives,usable_by_guess,in_defeated_by_guess,target,guess_defeats}, 
literate={:-}{{$\la\ $}}1 {not}{{$\naf$}}1 {\\el}{{$\in$}}1 {\\cup}{{$\cup$}}1 {:~}{{$\law$}}1,
morecomment=[l]{\%},
}
\DeclareCaptionStyle{ruled}{labelfont=normalfont,labelsep=colon,strut=off}%
\floatstyle{ruled}
\newfloat{listing}{tb}{lst}{}
\floatname{listing}{Listing}          

\newcommand{\ababaf}{\textsc{ababaf}}
\newcommand{\cegar}{\textsc{abasp}}
\newcommand{\clingo}{\textsc{clingo}}

\newcommand{\contrary}[1]{\overline{#1}}
\newcommand{\contraryempty}{\contrary{\phantom{a}}}

\newcommand{\naf}{{\it not}\,}
\newcommand{\la}{\leftarrow}

\newcommand{\cf}{\mathit{cf}}
\newcommand{\stable}{{\mathit{stb}}}
\newcommand{\adm}{\mathit{adm}}

\newcommand{\prf}{\mathit{prf}'}
\newcommand{\oldprf}{\mathit{prf}}
\newcommand{\comp}{\mathit{com}}

\newcommand{\stb}{{\mathit{stb}}}
\newcommand{\com}{\mathit{com}}
\newcommand{\grd}{\mathit{grd}}

\newcommand{\asms}{asms}
\newcommand{\theory}{\mathit{Th}}

\newcommand{\cl}{\mathit{cl}}

\newcommand{\CF}{\mathcal{F}}
\newcommand{\CG}{\mathcal{G}}

\newcommand{\compred}{$||\!\!\!\leadsto$}
\newcommand{\compclass}[1]{\compred-#1}

\newcommand{\Cred}{\textit{Cred}}

\title{Instantiations and Computational Aspects of \\Non-Flat Assumption-based Argumentation}

\author{
	Tuomo Lehtonen$^1$,
	Anna Rapberger$^2$,
	Francesca Toni$^2$, \\
	Markus Ulbricht$^3$,
	Johannes P.\ Wallner$^4$\\[1ex]
	$^1$University of Helsinki, Department of Computer Science\\
	$^2$Imperial College London, Department of Computing\\
	$^3$Leipzig University, ScaDS.AI\\
	$^4$Graz University of Technology, Institute of Software Technology\\[1ex]
	$^1$tuomo.lehtonen@helsinki.fi
	$^2$\{f.toni,a.rapberger\}@imperial.ac.uk, \\
	$^3$mulbricht@informatik.uni-leipzig.de, \\
	$^4$wallner@ist.tugraz.at	
}

\begin{document}

\maketitle

\begin{abstract}
	%Assumption-based argumentation (ABA) is a well established structured argumentation formalism. 
	Most existing computational tools for \emph{assumption-based argumentation (ABA)} focus on so-called \emph{flat} frameworks, disregarding the more general case. 
	In this paper, we study an instantiation-based approach for reasoning in possibly \emph{non-flat} ABA.
	We make use of a semantics-preserving translation between ABA and bipolar argumentation frameworks (BAFs). %where reasoning problems are computationally milder in BAFs than in ABA.
	%In order to provide the necessary formal underpinning, 
	%We analyze the size of the instantiated BAF and discuss how to avoid the construction of redundant arguments.
	%By utilizing compilability theory, we establish that the constructed BAF will, however, in general be of exponential size. 
	%We thus also examine classes of ABA that give rise to fewer arguments. 
	By utilizing compilability theory, we establish that the constructed BAFs will in general be of exponential size.
	To keep the number of arguments and computational cost low, we present three ways of identifying redundant arguments. Moreover, we identify fragments of ABA which admit a poly-sized instantiation.
	We propose two algorithmic approaches for reasoning in non-flat ABA;
	% The first utilizes answer set programming (ASP) for instantiating BAFs%, followed by solving reasoning on the constructed BAFs using  Boolean satisfiability (SAT)
	% .
	the first utilizes the BAF instantiation while %instantiation-based and 
	the %latter 
	second works directly without constructing arguments.
	%and employs modern declarative solvers.
	% The second %, for comparison, %we also propose 
	% is a direct ASP-based algorithm that does not explicitly construct arguments%, adapting state-of-the-art approaches
	% for many structured argumentation problems, %including flat ABA.
	An empirical evaluation shows that the former outperforms the %``direct'' ASP approach 
	latter on many instances, reflecting the lower complexity of BAF reasoning.
	This result is in contrast to flat ABA, where direct approaches dominate instantiation-based solvers. 
\end{abstract}

\section{Introduction}

Formal argumentation constitutes a %major 
prominent branch of %Artificial Intelligence (AI) 
AI that studies and 
develops computational approaches to reason argumentatively~\cite{arguHandbook}. The heterogeneity of %argumentation 
this field is reflected in various formalizations and %also 
application domains, such as legal reasoning, medical sciences, and e-democracy~\cite{AtkinsonBGHPRST17}. Computational approaches %built on solid foundations 
to solve key reasoning tasks are critical to the deployment of formal argumentation. %, which are 

Argumentation formalisms are often classified %into 
as either abstract or structured% argumentation
. %Whereas the former field
Abstract argumentation~\cite{Dung95} is concerned with acceptability of arguments based exclusively on the relations between %abstracted arguments
them. 
%Argumentation \emph{semantics} formalize jointly acceptable argument sets. %, the latter 
Structured argumentation formalisms~\cite{BesnardGHMPST14} capture an entire argumentative workflow~\cite{CaminadaA07}: starting from knowledge bases, a process of argument generation or instantiation is prescribed, upon which \emph{semantics} %(including but not limited to those for abstract argumentation) 
can be %applied 
deployed to find acceptable arguments or conclusions thereof.

Computational approaches to %the more involved formalisms in
structured argumentation %~\cite{BesnardGHMPST14} 
have gained increased attention in the research community. The biannual International Competition on Computational Models of Argumentation (ICCMA)~\cite{ThimmV17,GagglLMW20,LagniezLMR20,BistarelliKST21,DBLP:conf/kr/JarvisaloLN23} has recently included a dedicated track for assumption-based argumentation (ABA)~\cite{BondarenkoDKT97}---one of the prominent approaches to structured argumentation%---for the first time 
. %Moreover, 
Several algorithmic approaches to ABA or other well-known structured argumentation formalisms like ASPIC$^+$~\cite{ModgilP13} have been proposed recently~\cite{CravenT16,LehtonenWJ17,BaoCT17,KaramlouCT19,DBLP:conf/kr/LehtonenWJ20,LehtonenWJ21a,LehtonenWJ21b,DBLP:conf/clar/DillerGG21,DBLP:conf/kr/LehtonenWJ22,Thimm17a,DBLP:conf/comma/LehtonenWJ22,DBLP:conf/kr/LehtonenR0W23}.%\todo{check this list/add other refs} %for instance for ABA or other well-known approaches to structured argumentation like ASPIC$^+$~\cite{ModgilP13}. 

%Despite the apparent surge of research works,
Within this surge of algorithmic efforts, many %algorithmic
solutions focus on %somewhat basic 
restricted fragments of  %very general
structured argumentation. For instance, most algorithms for ABA (e.g. \cite{CravenT16,LehtonenWJ21a,DBLP:conf/clar/DillerGG21}) focus on the %so-called 
\emph{flat ABA} fragment, which imposes a strong restriction on the knowledge base by
excluding derivations of `assumptions'. %, i.e., imposes restrictions on the allowed 
%via `rules'. 
%The more general %non-flat 
The general
ABA language, not restricted to the flat fragment and referred to in this paper as \emph{non-flat ABA}, 
is able to capture more expressive settings such as auto-epistemic reasoning~\cite{BondarenkoDKT97} and multi-agent settings where merging information from different sources can result in a non-flat knowledge base~\cite{DBLP:conf/aaai/0001PRT24}.
%it has also found applications in capturing forms of (abstract) bipolar argumentation~\cite{DBLP:conf/prima/Cyras0T17}.
%As concerns the latter, for instance, when merging information from different agents/sources it might happen that information deemed factual by some agent is interpreted in terms of assumptions by another; in this case, the resulting merged ABAF would be %rendered 
%non-flat. 
More broadly, non-flat ABA can capture situations where dependencies between assumptions need to be taken into account, and is therefore strictly more expressive than flat ABA. 
%This is also mirrored in the higher computational complexity of the general case: almost all reasoning problems in non-flat ABA are computationally harder then in the flat ABA fragment (see, e.g., \cite{CyrasOKT21,DimopoulosNT02}).
However, the few algorithms for non-flat ABA heavily restrict the kind of ABA frameworks (ABAFs) to which they apply (in \cite{KaramlouCT19} to \emph{bipolar} ABA frameworks~\cite{DBLP:conf/prima/Cyras0T17}).
%
%The underlying 
A possible reason for the focus on %the 
flat ABA %fragment 
is the %involved 
lower computational complexity. As shown by~\citeauthor{DimopoulosNT02}~(\citeyear{DimopoulosNT02}) and~\citeauthor{CyrasHT21}~(\citeyear{CyrasHT21}), all major reasoning tasks exhibit a one level jump in the polynomial hierarchy when going from flat to non-flat ABA, which, again, reflects the increased expressiveness of the general (not-flat) case.
%In contrast to flat ABA, the task of deciding whether a conclusion is credulously accepted in ABA is 
%$\SigmaP{2}$-complete for complete and preferred semantics, $\DP_2$-complete for grounded, and $\NP$-complete for stable semantics.

In this paper we fill this gap and 
%We take up this opportunity and challenge and address
address computational %approaches to 
challenges of non-flat ABA.
In particular, we investigate algorithmic approaches to reasoning in non-flat ABA, establish new theoretical insights and %employ 
evaluate them in practice.
%On the other hand, we propose another approach inspired by state-of-the-art approaches.
%Our choice of algorithmic solutions is driven by computational and explainability considerations: we aim at computationally viable solutions that can support applications by cognitively tractable means~\cite{vcyras2021argumentative}.  

Algorithms %to 
for flat ABAFs can be divided into 
(i) instantiating a semantics-preserving argumentation framework (AF)~\cite{Dung95} and then performing the reasoning task in the AF; and 
(ii) direct approaches operating on the given ABAF without %argument 
instantiation. 
There are benefits to instantiation-based approaches, such as explainability in the form of more abstract, cognitively tractable explanations~\cite{vcyras2021argumentative}. 
For performance, direct approaches, specifically ones using modern declarative methods, such as answer set programming (ASP)~\cite{GelfondL88,Niemela99}, typically have an edge over instantiation-based solvers in structured argumentation (e.g.\ \cite{LehtonenWJ21a,DBLP:conf/kr/LehtonenR0W23}).
However, in the case of non-flat ABA this might plausibly not be the case.
A recent instantiation~\cite{DBLP:conf/aaai/0001PRT24} translates a given non-flat ABAF into a bipolar AF (BAF)~\cite{AmgoudCLL08} such that reasoning in the constructed BAF is computationally milder compared to the initial ABAF% (in the size of the constructed BAF)
.
This is in contrast to the classical translation of flat ABAFs into AFs, where the complexity of reasoning is as high for the AFs. 
% The BAF instantiation also promises to be advantageous as concerns explainability, by giving rise to more abstract, cognitively tractable explanations~\cite{vcyras2021argumentative}. 
%\TODO{motivate instantiation-based solvers? the above motivation is just from a computational point of view, but are there conceptual advantages?}

% On the other hand, in~\cite{DBLP:conf/kr/LehtonenR0W23} it has been shown that flat ABA frameworks can be instantiated by using only polynomially many arguments. We prove that this is not possible for non-flat ABA. 
% Thus for the instantiation-based solver to be competitive, it is essential to perform the argument computation as efficiently as possible.  

% Consequently, both approaches face novel barriers for non-flat ABA frameworks: the instantiation must be adjusted, and reasoning is significantly more complex than for flat ABA. 
% \todo{replace with we wanted to compile but can't, here's what we do instead}
In this work, we study the advantage this lower complexity can give for reasoning in non-flat ABAFs via instantiation to BAFs.
While ideally the resulting BAF would be of polynomial size, similarly to a recently proposed translation from flat ABA to AF~\cite{DBLP:conf/kr/LehtonenR0W23}, we show that this is impossible in non-flat ABA.
Instead, we investigate redundancies that can be eliminated to optimize the instantiation.
% In this paper, we address efficient BAF instantiation from non-flat ABA, including theoretical results regarding the BAF instantiation, the study of ABA fragments, as well as implementation.
% We also implement an ASP-based approach that does not construct arguments, adapted from state-of-the-art algorithms for other structured argumentation problems. 

In more detail, our main contributions are:
%     We present a result based on compilability theory, suggesting that an instantiation in BAFs cannot avoid an exponential number of arguments.
%     \hfill \textcolor{gray}{Section~\ref{subsec:lower bound}} \\
%     Motivated by the fact that reasoning in instantiated BAFs is milder than for non-flat ABAFs, we show how to efficiently instantiate BAFs. We present three notions of redundancy in argument generation that can be avoided.
%     \hfill \textcolor{gray}{Section~\ref{sec:feasible}}\\
%     Towards a greater reach for applications, we identify fragments of non-flat ABAFs with milder complexity of reasoning: atomic and additive non-flat ABAFs. 
%     \hfill \textcolor{gray}{Section~\ref{sec:fragments}}\\
%     Based on our findings, we propose two algorithmic approaches for reasoning in non-flat ABAFs. 
%     The first one efficiently instantiates a BAF using our redundancy notions via ASP, followed by reasoning on the BAF via SAT. 
%     \hfill \textcolor{gray}{Section~\ref{sec:ababaf}}\\
%     The second one performs reasoning directly on the given non-flat ABAF using iterative ASP calls, similarly to state-of-the-art approaches for other beyond-NP structured argumentation problems.
%     \hfill \textcolor{gray}{Section~\ref{sec:cegar}}\\
%     Lastly, we find empirically that %the 
%     both algorithms are competitive, with relative performance depending on the benchmark set. This contrasts with the current dominance of non-instantiation approaches for other structured argumentation formalisms.
%     \hfill \textcolor{gray}{Section~\ref{sec:experiments}}
\begin{itemize}
	\item We present a result based on compilability theory, suggesting that an instantiation in BAFs cannot avoid an exponential number of arguments. 
	\hfill \textcolor{gray}{Section~\ref{subsec:lower bound}}
	\item Motivated by the lower complexity in instantiated BAFs compared to non-flat ABAFs, we show how to efficiently instantiate BAFs. We present three redundancy notions for argument generation. % that can be avoided. 
	\hfill \textcolor{gray}{Section~\ref{sec:feasible}}
	\item Towards a greater reach for applications, we also
	identify fragments of non-flat ABA
	with milder complexity: atomic and additive non-flat ABAFs. 
	\hfill \textcolor{gray}{Section~\ref{sec:fragments}}
	\item We propose two algorithmic approaches for reasoning in non-flat ABAFs. 
	The first one efficiently instantiates a BAF via ASP using our redundancy notions, followed by SAT-based reasoning on the BAF. 
	\hfill \textcolor{gray}{Section~\ref{sec:ababaf}}
	
	The second one performs reasoning directly on the given non-flat ABAF using iterative ASP calls, similarly to state-of-the-art approaches for other beyond-NP structured argumentation problems.
	\hfill \textcolor{gray}{Section~\ref{sec:cegar}}
	\item We show empirically that %the 
	both algorithms are competitive, with relative performance depending on the benchmark set. This contrasts with the dominance of non-instantiation approaches for other structured argumentation formalisms.
	\hfill \textcolor{gray}{Section~\ref{sec:experiments}}
\end{itemize}

%In this paper we focus on non-flat ABAFs with atomic sentences. 
%For simplicity, throughout the paper we restrict attention to (non-flat) ABAFs where sentences are atomic.
We focus on complete, grounded, stable and a complete-based version of preferred semantics. 
In the technical appendix, we expand on proof details, encodings, and how to approach admissible-based semantics, which was shown to be more involved~\cite{DBLP:conf/aaai/0001PRT24}.% 
%For admissible-based semantics, a more involved instantiation is required~\cite{DBLP:conf/aaai/0001PRT24}. 
%Coverage of these semantics, proofs, encodings and further
%details can be found in~\cite{lehtonen2024instantiations}.

%-argumentation
%-structured (abstract)
%-computational approaches on structured gaining momentum: ICCMA first time, and recent works
%-focus on ABA/ASPIC+, however on flat LP fragments for ABA
%-non-flat has motivations
%-for general ABA, there is a clear barrier: complexity jump (cite)
%-we take up the challenge 
%-recent works show two directions: constructing args or without
%-very recently a very first BAF instantiation for non-flat was presented
%-appealing since reasoning on instantiatied BAF is simpler, ie, after argument construction better (contrast to flat ABA)
%-however: high complexity and high number of arguments are barriers
%-we present a thorough study of devising algorithmic approaches to non-flat ABA, going from theory to practice

%\TODO{outline:}
%\begin{itemize}
%	\item we study combinatorial properties of the BAF instantiation 
%	\item we discuss ways to represent arguments 
%	\item we show that one needs exponentially many arguments in general
%	\item we discuss redundancy notions and computational beneficial fragments
%	\item we encode, implement, and evaluate
%\end{itemize}

\section{Background}
\label{sec:background}
We recall preliminaries for assumption-based argumentation frameworks (ABAFs)~\cite{BondarenkoDKT97,CyrasFST2018}, bipolar argumentation frameworks (BAFs)~\cite{AmgoudCLL08}, instantiation of BAFs to capture ABAFs~\cite{DBLP:conf/aaai/0001PRT24} and basics of computational complexity%the reasoning problems of interest
. 
\paragraph{Assumption-based Argumentation.}
We assume a deductive system $(\mathcal{L},\mathcal{R})$, with  $\mathcal{L}$ restricted to a set of atoms and $\mathcal{R}$ a set of rules over $\mathcal{L}$. A rule $r \!\in\! \mathcal{R}$ has the form
$a_0 \!\leftarrow \! a_1,\ldots,a_n$ with $a_i \!\in \! \mathcal{L}$.
We denote the head of $r$ by $head(r) = a_0$ and the (possibly empty) 
body of $r$ with $body(r) = \{a_1,\ldots,a_n\}$.%
\begin{definition}
	An ABAF is a tuple $(\mathcal{L},\mathcal{R},\mathcal{A},\contraryempty)$, where $(\mathcal{L},\mathcal{R})$ is a \emph{deductive system}, $\mathcal{A} \subseteq \mathcal{L}$ is a non-empty set of \emph{assumptions}, and $\contraryempty:\mathcal{A}\rightarrow \mathcal{L}$ %the
	is a (total) \emph{contrary} function.
	%$\contraryempty$ is a function mapping assumptions $a\in \mathcal{A}$ to sentences $\mathcal{L}$. 
\end{definition}
Here, we focus on finite ABAFs, i.e., $\mathcal{L}$ and $\mathcal{R}$ are finite.
%\TODO{why don't we say: here from now on, unless stated otherwise,  we assume an ABAF D=.... so we shorten all defs and results there by omitting?}
%From now on, unless stated otherwise, we assume a fixed but arbitrary ABAF $D=(\mathcal{L},\mathcal{R},\mathcal{A},\contraryempty)$.

An atom $p \!\in\! \mathcal{L}$ is derivable from assumptions $S \!\subseteq\! \mathcal{A}$ and rules $R \!\subseteq \! \mathcal{R}$, denoted by $S \vdash_R p$,
if there is a finite rooted labeled tree $G$ such that the root is labeled with $p$, the set of labels for the leaves of $G$ is %equal to 
$S$ or $S \cup \{\top\}$, and
for every inner node $v$ of $G$ there is a rule $r \in R$ such that $v$ is labelled with $head(r)$, the number of %successors 
children of $v$
is $|body(r)|$ and every %successor 
child of $v$ is
labelled with a distinct $a \in body(r)$ or $\top$ if $body(r)=\emptyset$.
%We sometimes drop the subscript $R$ and write $S \vdash_R p$ simply as $S \vdash p$.
%We call such a derivation $S\vdash_R p$ 
We call such a $G$ a \emph{tree-based argument}. 
A \emph{sub-argument} of $G$ is a finite rooted labelled sub-tree $G'$ such that the leafs of $G'$ are a subset of the leafs of $G$. 
We will denote with $\args_D$ the set of all tree-based arguments of an ABAF~$D$. 
Following \citeauthor{Toni14}~(\citeyear{Toni14}), we will often compactly denote a tree-based argument with root $p$ and leafs $S$ as $S\vdash p$. % (thus disregarding the rules).

\iffalse Given that the semantics of ABA disregards $R$ in tree-based arguments $S \vdash_R p$ (see \cite{Toni14}), we will often compactly denote these arguments as follows:

\begin{definition}
	The set of \emph{tree-based arguments of} an ABAF~$D$ is
	$\args_D = \{ S\vdash  p \mid S\subseteq \mathcal{A},p\in\mathcal{L}, \exists \R\subseteq\mathcal{R}\mbox{ s.t.\ }S\vdash_R p\}$.
	%denote the set of \emph{core arguments} of an ABAF~$D$. 
\end{definition}
\fi

For %a set 
$S\!\subseteq\! \mathcal A$, we let 
$\contrary{S}\!=\!\{\contrary{a}\mid a\!\in \!S\}$.
By $\theory_D(S)\!=\!\{p \!\in \! \mathcal L\mid \exists \!S'\!\subseteq\! S:S'\vdash_R p\}$
we denote the set of all conclusions derivable from $S$. Note that $S\subseteq \theory_D(S)$ %since 
as
each %assumption 
$a\!\in \!\mathcal{A}$ is derivable via $\{a\}\!\vdash_\emptyset \!a$. 
For $S\vdash \!p\in \!\args_D$, we let 
%For derivation $S\vdash p$ we define the shorthand  
$\asms(S\vdash p)\!=\!S$; for %a set 
$E \!\subseteq \!\args_D$ %of arguments 
we let
$\asms(E)=\bigcup_{x\in E}\asms(x)$.

\begin{definition}
	\label{def:ABA defense}
	Let $D=(\mathcal{L},\mathcal{R},\mathcal{A},\contraryempty)$ be an ABAF, $S,T \!\subseteq \!\mathcal A$, and $a\in \mathcal{A}$.
	The \emph{closure} $\cl(S)$ of $S$ is $\cl(S) \!=\! \theory_D(S) \!\cap\! \mathcal A$. 
	%The set 
	$S$ is \emph{closed} iff $S=\cl(S)$;
	$S$ \emph{attacks}  $T$ iff $\contrary{b}\in\theory_D(S)$ for some $b\in T$;
	$S$ \emph{defends} $a$ iff for each closed %set
	$V\subseteq \mathcal A$ %of assumptions 
	s.t.\ $V$ attacks $a$, $S$ attacks $V$; 
	$S$ \emph{defends itself} iff $S$ defends each $b\in S$. 
\end{definition}
%Observe that, in order for $S$ to be non-closed, it is necessary that $\mathcal R$ contains a rule $r$ %with $head(r) \in \mathcal{A}$, \ie the head of $r$ SPACE SAVING CUT IF NEEDED
%whose head is an assumption. 
An ABAF where each set of assumptions is closed
is called \emph{flat}. We refer to an ABAF not restricted to be flat as \emph{non-flat}.

For $a\in \mathcal{A}$, we also say $S$ attacks $a$ if $S$ attacks the singleton $\{a\}$, and we write $\cl(a)$ instead of $\cl(\{a\})$.

%We now consider defense of assumption sets. 
%Now we consider defense~\cite{BondarenkoDKT97,CyrasFST2018}.  
%Observe that defense 
%In non-flat ABAFs defense only required against \emph{closed} sets of attackers. 
%Formally: 

A set $S\subseteq \mathcal{A}$ is \emph{conflict-free} in $D$, denoted $E\in\cf(D)$, if $E$ is not self-attacking;
$S$ is \emph{admissible}, denoted $S\in\adm(D)$, if $S$ is closed, conflict-free and it defends itself.
We focus on grounded, complete, $\subseteq$-maximal complete, and stable ABA semantics (abbr.\ $\grd$, $\comp$, $\prf$, $\stb$), as follows.\footnote{In the %main 
	paper we show results for $\grd$, $\com$, $\stable$ and $\prf$ defined as $\subseteq$-maximal complete sets (a variant to the usual $\mathit{prf}$ semantics given as $\subseteq$-maximal admissible sets). 
	We give results for $\adm$ and the standard $\mathit{prf}$ in~\cite{lehtonen2024instantiations}.
	%the supplementary material. %As usual in argumentation, we call the ``well-founded'' sets ``grounded''. 
	Although the intersection of all complete sets has been originally termed \emph{well-founded semantics}~\cite{BondarenkoDKT97}, we stick to the usual convention in the argumentation literature and call it ``grounded''.}

\begin{definition}
	\label{def:ABA semantics}
	Let $D=(\mathcal{L},\mathcal{R},\mathcal{A},\contraryempty)$ be an ABAF and $S \subseteq \mathcal{A}$ be a set of assumptions s.t.\ $S\in\adm(S)$. We say
	\begin{itemize}
		%\item $S\in\adm(D)$ iff $S$ is closed and defends itself,
		\item $S\in \comp(D)$ iff it contains every $T\subseteq\mathcal{A}$ it defends, 
		\item $S\in \prf(D)$ iff $S$ is $\subseteq$-maximal in $\com(D)$,
		\item $S\in \grd(D)$ iff $S = \bigcap_{T\in \com(D)} T$, and 
		\item $S\in \stb(D)$ iff $S$ attacks each $x \in \mathcal{A} \setminus S$.
	\end{itemize}
	%	Moreover, 
	%	\begin{itemize}
		%		\item $S\in\sstb(D)$ iff $S$ is conflict-free, closed, and attacks $\cl(x)$ for each $x \in \mathcal{A} \setminus S$; 
		%		\item $S\in \cprf(D)$ iff $S$ is $\subseteq$-maximal in $\com(D)$. 
		%	\end{itemize}
\end{definition}
In this paper we stipulate that the empty intersection is interpreted as $\emptyset$, %\ie 
thus if $\com(D) = \emptyset$, then 
$\grd(D) = \emptyset$.

\begin{example}\label{ex:intro}
	We let $D$ be an ABAF 
	with assumptions $\mathcal A = \{a,b,c\}$, rules 
	$\mathcal R = \{(p\gets a), (q\gets b), (c \gets p,q)\}$, and $\contrary{b}=p$. 
	The set $\{a,c\}$ is ($\subseteq$-maximal) complete and stable in $D$ since $\{a,c\}$ is unattacked and $\{a\}$ attacks $b$.
	Note that $\{a,b\}$ is not closed since %$\{a,b\}\vdash c \in \args_D$.\todo{not quite the definition, use cl instead?}
	$\cl(\{a,b\})= \{a,b,c\}$.
\end{example}

\paragraph{Bipolar Argumentation.}
Bipolar argumentation features adversarial and supportive relations between arguments.
%~\citeauthor{Dung95}.
%We recall the definitions following~\cite{DBLP:conf/aaai/0001PRT24}.
\begin{definition}
	A \emph{bipolar argumentation framework (BAF)} $\CF$ is a tuple of the form $\CF = (\args, \attacker, \supporter)$ 
	where 
	$\args$ represents a set of arguments,
	$\attacker\subseteq \args\times \args$ models \emph{attack}, and $\supporter \subseteq \args \times \args$ models \emph{support} between them.
\end{definition}
Thus, in the tradition of Dung's abstract \emph{argumentation frameworks (AFs)}~(\citeyear{Dung95}),  arguments in BAFs are considered abstract, but
BAFs extend AFs by integrating support.
We call $F = (\args,\attacker)$ the {underlying} AF of $\CF= (\args, \attacker, \supporter)$.
%Graphically, we depict attacks by solid and support by dashed edges.

For two arguments $x,y\in \args$, if $(x,y)\in \attacker$ ($(x,y)\in \supporter$) we say that $x$ \textit{attacks} (\emph{supports}) $y$ as well as $x$ \textit{attacks} (\emph{supports}) (the set) 
$E\subseteq \args$ given that $y\in E$.

We utilize the BAF semantics introduced by~\citeauthor{DBLP:conf/aaai/0001PRT24}~(\citeyear{DBLP:conf/aaai/0001PRT24}).
For a set $E\subseteq \args$, we let $\cl(E) = E\cup \{ a\in A \mid \exists e\in E:\; (e,a)\in\supporter \}$. 
The set $E$ is \emph{closed} if $E = \cl(E)$; 
$E$ is \emph{conflict-free} in $\CF$, denoted $E\in\cf(\CF)$, if for no $x,y\in E$, $(x,y)\in \attacker$;
$E$ \emph{defends} $a\in \args$ if $E$ attacks each closed set $S\subseteq \args$  which attacks $a$.
The \emph{characteristic function} of $\CF$ is $\Gamma(E) = \{ a\in A\mid E\text{ defends }a \}$. 
We say $E$ is \emph{admissible}, denoted $E\in\adm(\CF)$, if $E$ is closed, conflict-free, and $E\subseteq \Gamma(E)$.
\begin{definition} \label{def:extsem}
	Let $\CF$ be a BAF. For a set $E\in \adm(\CF)$,
	\begin{itemize}
		%\item $E\in\adm(\CF)$ iff $E \subseteq \Gamma(E)$; 
		\item $E\in\com(\CF)$ iff $E = \Gamma(E)$; 
		\item $E\in\grd(\CF)$ iff $E = \bigcap_{S\in\com(\CF)} S$; 
		\item $E\in\prf(\CF)$ iff $E$ $\subseteq$-maximal in $\com(\CF)$;
		\item $E\in\stb(\CF)$ iff $E$ attacks all $a\in A\setminus E$.
	\end{itemize}
	%		$E\in \semi(\F)$ iff $E\in \com(\F)$ and there is no $D\in\com(\F)$ with $E\cup E^+_F\subset D\cup D^+_F$.
	%\end{enumerate}
\end{definition}
%We note that our definition of preferred semantics deviates from the standard definition which maximises over admissible instead of complete sets. \todo{unclear why we should care here}

%\begin{definition}
%    \label{def:pbaf-semantics}
%		For a pBAF $\PF$, a closed set 
%		$E\in\cf(\PF)$ is  
%		\begin{itemize} 
	%			%\item \emph{admissible}, $E\!\in\!\adm(\PF)$, iff $E$ is exhaustive and $E\!\subseteq\! \Gamma(E)$; 
	%			\item \emph{preferred}, $E\in\prf(\PF)$, iff it is $\subseteq$-maximal admissible; 
	%			\item \emph{complete}, $E\in\com(\PF)$, iff $E = \Gamma(E)$; 
	%			\item \emph{grounded}, $E\in\grd(\PF)$, iff $E = \bigcap_{S\in\com(\PF)}S$;
	%			\item \emph{stable}, $E\in\stb(\PF)$, iff $E_\CF^+ = A\setminus E$. 
	%		\end{itemize}
%\end{definition}

%\begin{definition} \label{def:extsem}
%	Let $\CF$ be a BAF. A closed set $E\in\cf(\CF)$ is
%	\begin{itemize}
	%		\item $E\in\com(\CF)$ iff $E = \Gamma_(E)$; 
	%		\item $E\in\grd(\CF)$ iff $E = \bigcap_{S\in\com(\CF)} S$ 
	%		\item $E\in\stb(\CF)$ iff $E^+=A\setminus E$.
	%	\end{itemize}
%		$E\in \semi(\F)$ iff $E\in \com(\F)$ and there is no $D\in\com(\F)$ with $E\cup E^+_F\subset D\cup D^+_F$.
%\end{enumerate}
%\end{definition}

\paragraph{ABA and Bipolar Argumentation.}
Each ABAF can be captured as a BAF as follows \cite{DBLP:conf/aaai/0001PRT24}.  
\begin{definition}
	\label{def:pbaf instantiation}\label{def:instantiated BAF}
	For an ABAF  
	$D = (\mathcal{L},\mathcal{R},\mathcal{A},\contraryempty)$, 
	the %corresponding 
	\emph{{instantiated} BAF} 
	$\CF_D = (\args,\attacker,\supporter)$ 
	is given by $A=\args_D$ and
	\begin{align*}
		%A &= \{ (S\vdash p) \mid (S\vdash p)\text{ is a  tree-based argument in }D \},\\
		\attacker  &= \{ ( S\vdash p, T\vdash q ) \in A\times A \mid p \in\contrary{T} \},\\
		\supporter &= \{ ( S\vdash p, \{a\} \vdash a ) \in A \times A \mid a \in\cl(S) \}.
	\end{align*} 
\end{definition}
\begin{example}
	\label{ex:intro inst}
	The ABAF from Example~\ref{ex:intro} yields the following BAF, with attacks as solid and supports as dashed lines, and arguments depicted as trees (with root at the top). 
	%Below, we depict the corresponding tree-based arguments. 
	\begin{center}
		\begin{tikzpicture}
			\node[draw,label={above:$A_1$}] (arg1a) at (5.5,0.65) {
				\begin{tikzpicture}[xscale=0.8,yscale=0.5]
					\node[targ] (p) at (0,0) {$p$};
					\node[targ] (q) at (0,-1.3) {$a$};
					
					\path[-]
					(p) edge (q)
					;
				\end{tikzpicture}
			};
			\node[draw,label={above:$A_2$}] (arg1b) at (6.5,0.35) {
				\begin{tikzpicture}[xscale=0.8,yscale=0.5]
					\node[targ] (p) at (0,0) {$q$};
					\node[targ] (q) at (0,-1.3) {$b$};
					
					\path[-]
					(p) edge (q)
					;
				\end{tikzpicture}
			};
			\node[draw,label={above:$A_3$}] (arg2) at (0.75,0.35) {
				\begin{tikzpicture}[xscale=0.8,yscale=0.5]
					\node[targ] (c) at (0,1.3) {$c$};
					\node[targ] (p) at (-.7,0) {$p$};
					\node[targ] (q) at (.7,0) {$q$};
					\node[targ] (a) at (-.7,-1.3) {$a$};
					\node[targ] (b) at (.7,-1.3) {$b$};
					
					\path[-]
					(a) edge (p)
					(b) edge (q)
					(p) edge (c)
					(q) edge (c)
					;
				\end{tikzpicture}
			};
			
			\node[draw,label={above:$A_4$}] (arga) at (2.5,1) {
				$a$
			};
			\node[draw,label={above:$A_5$}] (argb) at (3.5,1) {
				$b$
			};
			\node[draw,label={above:$A_6$}] (argc) at (4.5,1) {
				$c$
			};
			
			\path[->]				
			(arg1a) edge[dotted,out=180,in=-50] (arga)
			(arg1b) edge[dotted,out=230,in=-90,looseness=0.7] (argb)
			%die von A_3: 
			(arg2) edge[dotted] (arga)
			(arg2) edge[dotted,out=0,in=-130] (argb)
			(arg2) edge[dotted,out=-10,in=-100] (argc)
			(arg1a) edge[out=30,in=130] (arg1b)
			(arg1a) edge[out=200,in=-50] (argb)
			(arg1a) edge[out=220,in=-20] (arg2)
			;
		\end{tikzpicture}
	\end{center}
\end{example}
This BAF construction captures semantics for ABAFs, as recently shown~\cite{DBLP:conf/aaai/0001PRT24}.
\begin{theorem}
	\label{thm:SemanticsCorrespondenceADM}
	\label{th:semantics correspondence 2}
	Let $D$ be an ABAF,
	$\CF_D  = (\args,\attacker,\supporter)$ the %corresponding 
	associated BAF, and $\sigma\in\{\com,\prf,\grd,\stb\}$. Then  
	\begin{itemize}
		\item if $E\in\sigma(\CF_D)$, then $\asms(E)\in\sigma(D)$, and 
		\item if $S\in\sigma(D)$, then $\{ x\in A \mid \asms(x)\subseteq S \}\in\sigma(\CF_D)$.
	\end{itemize}
\end{theorem}

\paragraph{Computational Complexity.}
We assume the reader to be familiar with the polynomial hierarchy. 
We focus on the credulous reasoning task. 
For a BAF $\CF$ and a semantics $\sigma$, an argument $a\in \args$ is \emph{credulously accepted} if $a\in E$ for some $E\in\sigma(\CF)$;
for an ABAF $D$ and a semantics $\sigma$, a conclusion $p\in \mathcal{L}$ is \emph{credulously accepted} if $p\in \theory_D(E)$ for some $E\in\sigma(D)$.
The induced decision problems are denoted by $\Cred^{BAF}_{\sigma}$  and $\Cred^{ABA}_{\sigma}$, respectively.
We note that %credulous reasoning w.r.t.\ preferred semantics coincides with credulous reasoning w.r.t.\ complete semantics, i.e., 
$\Cred^{\mathcal{C}}_\com=\Cred^{\mathcal{C}}_{\prf}$ for $\mathcal{C}\in\{BAF,ABA\}$.

Deciding credulous reasoning for non-flat ABA is $\SigmaP{2}$-complete for complete, $\DP_2$-complete for grounded, and $\NP$-complete for stable semantics~\cite{DimopoulosNT02,CyrasHT21}. As recently shown~\cite{DBLP:conf/aaai/0001PRT24}, the corresponding decision problems for BAFs exhibit lower complexity for main semantics: 
%are one level lower in the polynomial hierarchy, for all except stable semantics, i.e., it is 
$\DP$-completeness for grounded and $\NP$-completeness for complete and stable.
%\todo{actually the existing results are for the original definitions in ABA or BAFs, so we are not correct hedre, but let us ignore and hope that nobody notices...}

\section{BAF Generation in Theory}
\label{sec:size theory}
For instantiation-based algorithms the number of arguments is critical for the run time performance. 
Direct instantiation methods which compute all tree-based arguments may yield unfeasibly large argumentation graphs. 
This has already been pointed out for flat ABA, which motivated the study of means to reduce their size~\cite{DBLP:conf/kr/LehtonenR0W23}. 
In this section, we analyze redundancies for the non-flat case as well. 
%Our main goal is to identify redundancies we can exploit in order to keep the solver as efficient as possible. 
However, our first observation is that the computation of exponentially many arguments can in general not be avoided for non-flat ABA. 
A result of this kind is, to the best of our knowledge, novel for structured argumentation. 

%\subsection{Combinatorial Foundations}
%
%\TODO{@Anna: This is copy \& paste from KR paper, can you make this work for non-flat?}
%The number of tree-based arguments that can be constructed from a given ABAF $D$ depends on the number of rules, derivation-depth, rule-size, and number of rules with the same heads, as follows.
%\begin{restatable}{proposition}{thUpperBoundArgs}
%	\label{thm:upper-bound-args}
%	For each $m$-rule-size bounded ABA framework $D=(\mathcal{L},\mathcal{R},\mathcal{A},\contraryempty)$ 
%	with $\vert \{r\in\mathcal{R}\mid head(r)=s\}\vert \leq l$ for all $s\in\mathcal{L}$,
%	there are at most 
%	\begin{align*}
	%		l^p\cdot |\mathcal{L}\setminus \mathcal{A}|,\quad \mbox{ with } p=\sum_{i=0}^{k-1}m^i,
	%	\end{align*}
%	many tree-based arguments of height $k\geq 1$.
%\end{restatable}

\subsection{A Lower Bound For Non-Flat Instantiations}
\label{subsec:lower bound}

We give a formal line of reasoning that it is impossible to instantiate a given non-flat ABAF in a ``reasonable'' way and thereby obtain a polynomial-sized %argumentation 
graph (\ie AF or BAF). 

In more detail, we show a complexity result based on compilability theory~\cite{CadoliDLS02} stating that, unless the polynomial hierarchy collapses, one cannot transform a given non-flat ABAF $D$ into some structure $\chi$ with the following properties: 
i) $\chi$ is of polynomial size w.r.t.\ $D$; and 
ii) in $\chi$ one can decide in polynomial time
whether a given set $E$ of assumptions is admissible or complete in $D$. 
Since verifying admissible sets in non-flat ABA is $\NP$-complete, it is clearly impossible to construct such $\chi$ in polynomial time. However, we do not need this restriction. That is, even given exponential time, such $\chi$ cannot be constructed.  
Conceptually, this result excludes usual argument-centric instantiations of $D$ into AFs or BAFs, because here we expect checking whether a set of arguments is admissible to be tractable. 

That is, under complexity theoretic assumptions, our result states that it is impossible to apply instantiations of $D$ into a polynomial-sized AF or BAF which have the usual correspondences between sets of arguments and sets of assumptions. 

% if one can instantiate a given non-flat LP ABA framework such that one can have a polynomial-sized structure and from this structure one can in polynomial-time decide whether a given set of assumptions is admissible, such a result would collapse the polynomial hierarchy. For instance, if there is an instantiation of the ABA framework to some set of arguments and attacks (and supports) such that one can find all arguments that can be composed from a set of assumptions one can in AFs and BAFs decide directly whether this set of arguments is admissible (if the AF or BAF is computed). 

% \TODO{todos:}
% \begin{itemize}
	% 	\item see theory of compilability \cite{CadoliDLS02}
	% 	\item give reader an introduction: what is going here? 
	% 	\item check proof and finalize everything
	% 	\item conclude: good for instantiation: complexity difference; bad for instantiation: number cannot be reduced, as we just saw
	% \end{itemize}

\begin{restatable}{theorem}{thLowerBoundBAF}
	\label{th:lower bound baf}
	Unless the polynomial hierarchy collapses one cannot transform a non-flat ABAF into a polynomial-sized AF or BAF from which one can in polynomial-time decide whether a given set of assumptions is admissible or complete.
\end{restatable}

% \begin{proof}
	% TODO: essentially from~\cite{CadoliDLS02}.
	% \end{proof}

% If I understood it correctly, then this result might imply that one cannot transform a non-flat LP ABA framework to some structure where one can decide whether a given assumption set is admissible in polytime. If we assume that one can construct an AF with poly-space bound that ``somehow'' represents the non-flat ABA framework, then this might imply that one cannot directly check whether a set of assumptions is admissible (but then this does not correspond to some set of arguments somehow directly, since adm checking is in P in AFs). 

\subsection{Towards Feasible BAF Instantiations}
\label{sec:feasible}
The previous subsection shows that instantiation of non-flat ABAFs will require exponentially many arguments in general. 
We strive to construct as few as possible nonetheless. 
We identify three redundancy notions to reduce the number of arguments. %that give rise to certain types of arguments which do not need to be constructed. 
%
%
%We note that all redundancy definitions below do not make use of the premise function, hence they are identical for BAFs and pBAFs. 
%Since BAFs can be considered as special case of pBAFs for complete-based semantics, all results in this section apply for both BAFs and pBAFs. 
We consider an arbitrary but fixed semantics $\sigma\in \{ \com,\grd,\stb,\prf\}$ throughout this subsection.

\paragraph{Derivation Redundancy.} 
We call the first notion derivation redundant arguments as it spots ``inefficient'' derivations. 
\begin{definition}
	\label{def:redundant}
	For an ABAF $D$
	and its set of arguments $A_D$,
	%and its %BAF $\CF_D = (\args,\attacker,\supporter)$ 
	%BAF $\CF = (\args,\attacker,\supporter)$,
	we call an argument 
	$(S\vdash p)\in A_D$ \emph{derivation redundant} iff there is an argument 
	$(S'\vdash p)\in A_D$ 
	with $S'\subsetneq S$. 
\end{definition}
%TO ADD POSSIBLY LATER: an argument is non-derivation redundant iff it is support minimal as defined in https://link.springer.com/chapter/10.1007/978-3-642-54373-9_4
\begin{example}
	Recall Example~\ref{ex:intro}. %with rules 
	%$\mathcal R = \{p\gets a.,\; q\gets b.,\; c \gets p,q.,\}$
	Suppose we consider an ABAF $D'$ by setting 
	$\mathcal R' = \mathcal R \cup \{ c\gets p \}$. 
	Then we would obtain a new argument $\{a\}\vdash c$. 
	Then the existing argument $A_3$ representing 
	$\{a,b\}\vdash c$ becomes derivation redundant. 
\end{example}

%%%%%%%%%%%%%%%%%% BACK IN FOR JOURNAL %%%%%%%%%%%%%%%%%%%%%%%%%
%Before proving Proposition~\ref{prop:redundant args}, let us first note that the closure function within any BAF is \emph{additive}, \ie we have %if 
%$\cl(E) = \bigcup_{e\in E}\cl(\{e\})$
%for any set $E$ of arguments. 
%Note that this is in general not the case of a non-flat ABAF $D$ (but this constitutes an interesting fragment as we discuss in Section~\ref{sec:additive} below). 
%Hence, the closure function $\cl$ becomes simpler when instantiating the BAF $\CF_D$. 
%\begin{restatable}{lemma}{leAdditiveClosure}
%	\label{le:additive closure}
%	Let $\CF =(\args,\attacker,\supporter)$ be a BAF. 
%	Then for any set $E\subseteq A$ of arguments we have $\cl(E) = \bigcup_{e\in E}\cl(\{e\})$. 
%\end{restatable} 
%\begin{proof}
%	This follows directly from the existential quantifier in the definition of the closure operator $\mu$. 
%\end{proof}

We observe that derivation redundant arguments can be removed without altering the sets of accepted assumptions.% 
\begin{restatable}{proposition}{propRedundantArgs}
	\label{prop:redundant args}
	Let $D$ be an ABAF and $\CF_D$ the corresponding BAF. 
	Let $x\in A_D$ be derivation redundant and let $\CG$ be the BAF after removing the argument $x$ from $\CF_D$. 
	Then 
	$$ \{ \asms(E)  \mid E\in\sigma(\CF_D) \} = \{ \asms(E)  \mid E\in\sigma(\CG) \}. $$
\end{restatable}
%By setting $S = S'$ in Definition~\ref{def:redundant} we obtain that it suffices to construct one representative $S\vdash p$ of each tree-based argument with leaf notes $S$ and root $p$. 
%This motivates the following notion of \emph{core} arguments, similar in spirit to core arguments for flat ABA \cite{DBLP:conf/kr/LehtonenR0W23}.  
%\begin{definition}
%	Given an ABAF $D$, we call 
%	$\args = \{ (S, p) \mid S\vdash p \textnormal{ is an argument in } D \}$
%	the set of \emph{core arguments} of $D$. 
%\end{definition}

\paragraph{Expendable Arguments.} 
We derive another redundancy notion based on the conclusion of arguments: if $x = (S\vdash p)$ is an argument where $p$ is neither an assumption nor a contrary, then $x$ merely represents an intermediate derivation step. 
Arguments of this kind do not need to be instantiated if one is interested in assumption extensions only. 
\begin{definition}
	For an ABAF 
	%$D=(\mathcal{L},\mathcal{R},\mathcal{A},\contraryempty)$ 
	$D$
	and its set of arguments $A_D$
	%and its BAF $\CF_D = (A,\attacker,\supporter)$ 
	we call an argument 
	$(S\vdash p)\in A_D$ \emph{expendable} iff 
	$p\notin \mathcal A \cup \contrary{\mathcal A}$.
\end{definition}
\begin{example}
	In our Example~\ref{ex:intro}, 
	for instance the argument $A_2$ representing $\{b\}\vdash q$ is expendable since $q$ is neither a contrary nor an assumption.  
\end{example}
Since arguments of this kind have no out-going attacks, they do not contribute to the semantics of the instantiated BAF.
However, keep in mind that we still construct relevant super-arguments of expendable ones. 
\begin{restatable}{proposition}{propExpendableArgs}
	\label{prop:expendable args}
	Let $D$ be an ABAF and $\CF_D$ the corresponding BAF. 
	Let $x\in A_D$ be an expendable argument and let $\CG$ be the BAF after removing the argument $x$ from $\CF_D$. 
	Then 
	$$ \{ \asms(E)  \mid E\in\sigma(\CF_D) \} = \{ \asms(E)  \mid E\in\sigma(\CG) \}. $$
\end{restatable}

\paragraph{Assumption Redundancy.}
This final redundancy notion is specific to non-flat ABAFs. 
It states that %tree-based 
arguments that make use of assumptions in intermediate steps can be neglected. 
This requires rules with assumptions in their head and is thus not possible for flat ABA. 

\begin{definition}
	\label{def:assumption redundant}
	For an ABAF 
	$D=(\mathcal{L},\mathcal{R},\mathcal{A},\contraryempty)$
	and its set of arguments $A_D$,
	%BAF $\CF_D = (A,\attacker,\supporter)$ 
	%we call an argument $(S\vdash p)\in A_D$ \emph{assumption redundant} if there is a corresponding 
	an % tree-based 
	argument $x=(S\vdash p)\in A_D$ is \emph{assumption redundant} iff it contains a sub-argument $x'$ s.t.\
	%it contains a sub-argument $x'$ s.t.\ 
	%
	\begin{itemize}
		\item $x'$ is a proper sub-argument of $x$, \ie $x'\neq x$, 
		\item $x$ is of the form $S'\vdash a$ where $S'\subseteq S$ and 
		$a\in\mathcal{A}$. 
	\end{itemize}
\end{definition}
\begin{example}
	Suppose we augment Example~\ref{ex:intro} with the additional rule ``$a\gets b$''. This would lead to a novel argument $A_7$ for $p$ by first applying the rule ``$a\gets b$'' and then ``$p\gets a$''. 
	However, since $a$ is an assumption itself, it is more efficient to infer $p$ from $a$ directly, which is represented by the argument $A_1$. 
	Thus $A_7$ would be assumption redundant. 
\end{example}
As for the other redundancy notions, arguments of this kind can be removed without altering the semantics. 
\begin{restatable}{proposition}{propAssumptionRedArgs}
	\label{prop:assumption redundant args}
	Let $D$ be an ABAF and $\CF_D$ the corresponding BAF. 
	%Let $\sigma\in \{ \com,\grd,\stb ,\adm,\prf\}$.
	Let $x \in A_D$ be an assumption redundant argument and let $\CG$ be the BAF after removing $x$ from $\CF_D$. 
	Then 
	$$ \{ \asms(E)  \mid E\in\sigma(\CF_D) \} = \{ \asms(E)  \mid E\in\sigma(\CG) \}. $$
\end{restatable}

\paragraph{Summary.}
With our redundancy notions, the instantiated BAF can be reduced as follows. 
From $A_D$ we construct the set $A^*_D$ of \emph{non-redundant} arguments by
i) first removing all derivation redundant arguments from $A_D$; 
ii) then removing all expendable arguments from the result of i); and finally 
iii) removing all assumption redundant arguments from the result of ii). 
Now we define the redundancy-free core of $D$. 
\begin{definition}
	Let $D=(\mathcal{L},\mathcal{R},\mathcal{A},\contraryempty)$ be an ABAF. 
	The BAF $\CG=(\args,\attacker,\supporter)$ is the \emph{non-redundant core} of $D$ where 
	$\args = \{ (S, p) \mid S\vdash p \textnormal{ is an argument in } A^*_D \}$, 
	$\attacker$ is the set of all attacks between arguments in $\args$, and 
	$\supporter$ is the set of all supports between arguments in $\args$. 
\end{definition}
Due to Propositions~\ref{prop:redundant args},~\ref{prop:expendable args}~and~\ref{prop:assumption redundant args} this representation is semantically equivalent. 
\begin{corollary}
	\label{cor:efficient instantiation}
	Let $D$ be an ABAF and $\CF_D$ its corresponding BAF. 
	If $\CG$ is the redundancy-free core of $D$, then we have that
	$ \{ \asms(E)  \mid E\in\sigma(\CF_D) \} = \{ \asms(E)  \mid E\in\sigma(\CG) \} $. 
\end{corollary}
By our previous results, the non-redundant core preserves the semantics of the given ABAF, since only redundant arguments are omitted and the representation streamlined. 
By applying this representation, we still have exponentially many arguments in general, but finiteness is guaranteed.  
\begin{restatable}{proposition}{propNumberOfCoreArguments}
	The redundancy-free core $\CG$ of an ABAF $D$ has at most $|2^\mathcal A|\cdot |\mathcal L|$ arguments.
\end{restatable}

%\TODO{what if an atom $q$ is queried?}
%Given an ABAF $D$ and a queried atom $q\in\mathcal{L}$, each core argument $(A,s)$ such that there is no $A'\subset A$ where $A'\vdash s$.
%For credulous acceptance, only construct arguments for atoms that are a contrary of some assumption, or the query.
%\red{TODO: this works for BAF as well? What about skeptical?}
%Attacks and supports are created as defined.

%\section{Clustering Arguments}
%\label{sec:cluster args}
%\red{This is (probably) obsolete}
%
%Given an ABAF $D$ and a queried atom $q\in\mathcal{L}$, the following c-arguments (clustered argument) are generated.
%For a set of assumptions $S\subseteq\mathcal{A}$ we have a c-argument $A_S$ with $A_S\models a$ for each $a\in\theory(S)$ and $asms(A_S)=S$ if
%\begin{itemize}
%    \item either $\overline{a}\in\theory(S)$ for any $a\in\mathcal{A}$, or $q\in\theory(S)$, or $a\in\theory(S)$ s.t. $a\in\mathcal{A}\setminus S$ \red{(this condition needs to be removed or revised for the non-flat case if I recall correctly?)},
%    \item for each assumption $a\in S$, $S'\vdash l$ for some $l\in\theory(S)$ and $a\in S'\subseteq S$ (each assumption is used in some derivation), and
%        %for each assumption $a\in S$, there is a derivation tree for some atom in $\theory(S)$ that has $a$ as a leaf, and
%    \item there is some atom $l\in\theory(S)$ such that $l\notin\theory(S')$ for any $S'\subset S$.
%\end{itemize}
%In addition, for any $S$ failing the last condition, a c-argument may be constructed.
%Finally, we have the singleton arguments for each assumption.

\subsection{Fragments} 
\label{sec:fragments}
As we just saw, non-flat ABA instantiations will in general have exponentially many arguments. 
In this subsection, we investigate fragments inducing fewer arguments. 
Since reasoning in BAFs is milder than in non-flat ABAFs, %~\cite{DBLP:conf/aaai/0001PRT24}, there is hope that 
we expect such fragments to admit lower complexity. 
Indeed, if the core computation is polynomial, then the complexity drops. 
\begin{restatable}{proposition}{propComplexityPolyArgs}
	\label{prop:complexity poly args}
	Let $\mathcal C$ be a class of ABAFs s.t.\ the non-redundant core of $D$ can be computed in polynomial time. 
	Then, the computational complexity of reasoning problems in $D$ is not harder than in the %constructed 
	instantiated BAF, \ie 
	%	\begin{itemize}
		%		\item 
		%		$\Ver_{\sigma}$ is 
		%		tractable for $\sigma\in\{\adm,\com,\stb\}$, 
		%		in $\coNP$ for $\sigma = \prf$, and  
		%		in $\DP$ for $\sigma = \grd$. 
		%		\item 
		$\Cred^{ABA}_{\sigma}$ is 
		in $\NP$ for $\sigma\in\{\adm,\com,\prf,\stb\}$ and 
		in $\DP$ for $\sigma = \grd$. 
		%		\item 
		%		$\Skept_{\sigma}$ is 
		%		in $\coNP$ for $\sigma = \adm$, 
		%		in $\DP$ for $\sigma \in \{\com,\grd,\stb\}$, and
		%		in $\Pi^P_2$ for $\sigma = \prf$.  
		%	\end{itemize}
\end{restatable}
Indeed, in our empirical evaluation in Section~\ref{sec:experiments} we will see that these fragments are milder in practice as well. 

\paragraph{Atomic ABAFs}
The first fragment we consider is \emph{atomic}, which has been studied for flat ABA as well~\cite{RapbergerU23,DBLP:conf/kr/LehtonenR0W23}.
For an ABAF to be atomic, each rule body element has to be an assumption. 
\begin{definition}
	Let $D=(\mathcal{L},\mathcal{R},\mathcal{A},\contraryempty)$ 
	be an ABAF. 
	A rule $r\in\mathcal R$ is called \emph{atomic} if $body(r)\subseteq \mathcal A$. 
	The ABAF $D$ is called \emph{atomic} if each rule $r\in \mathcal R$ is atomic. 
\end{definition}
For flat ABA, this means that $D$ has $|\mathcal R| + |\mathcal A|$ arguments (each rule induces exactly one tree-based argument)~\cite{DBLP:conf/kr/LehtonenR0W23}. 
In contrast, the same is not immediate for non-flat ABA as the derivation of assumptions is allowed. %, so this does not necessarily hold anymore. 
Nevertheless, due to our notion of assumption redundancy from Definition~\ref{def:assumption redundant}, we can show that the number of non-redundant arguments is indeed linear in $D$. 
\begin{restatable}{proposition}{propArgNumberAtomic}
	The non-redundant core of atomic ABAFs 
	$D$ 
	%$D=(\mathcal{L},\mathcal{R},\mathcal{A},\contraryempty)$ 
	consists of at most $|\mathcal{R}|+|\mathcal{A}|$ many arguments.
\end{restatable}
\begin{example}
	Consider an ABAF $D$ with 
	$\mathcal A = \{a,b,c\}$, 
	%$\mathcal L = \mathcal A \cup \{p,q\}$, 
	rules 
	$\mathcal R = \{(p\gets a), (q\gets b), (a \gets c)\}$, and an arbitrary contrary function. 
	%This ABAF 
	$D$ is atomic since each rule body consists of assumptions. 
	While it is possible to construct more that $|\mathcal R|$ + $|\mathcal A|$ %rules 
	arguments (by applying ``$a\gets c$'' and then ``$p\gets a$''), %this additional rule is 
	these additional arguments are 
	assumption redundant.
\end{example}
As a consequence, Proposition~\ref{prop:complexity poly args} applies to the class of atomic ABAFs, inducing lower complexity of reasoning. 
%\begin{corollary}
%    $\Cred^{ABA}_{\sigma}\in\NP$ for $\sigma\in\{\adm,\com,\prf,\stb\}$; $\Cred^{ABA}_{\sigma}\in\DP$ for $\sigma = \grd$. 
%\end{corollary}

\paragraph{Additive Closure}
\label{sec:additive}
In an atomic ABAF, each non-redundant argument has a derivation depth of $1$. 
We %analogously 
can also
bound the body size of rules, obtaining so-called \emph{additive} ABAFs.
\begin{definition}
	\label{def:additive}
	We call an ABAF 
	$D=(\mathcal{L},\mathcal{R},\mathcal{A},\contraryempty)$ 
	\emph{additive} 
	if for each rule $r\in \mathcal R$ it holds that $|body(r)|\leq 1$. 
\end{definition}
The reason for calling ABAFs of this kind additive is that they induce an additive $\theory_D$ mapping, that is, for each set $S$ of assumptions, 
$\theory_D(S) = \bigcup_{s\in S}\theory_D(s)$. 
\begin{example}
	Consider an ABAF $D$ with 
	$\mathcal A = \{a,b,c\}$, 
	$\mathcal L = \mathcal A \cup \{p,q,r,s\}$, 
	rules 
	$\mathcal R = \{(p\gets a), (q\gets b), (r \gets p), (s \gets q)\}$, and an arbitrary contrary function. 
	This ABAF is additive since each rule body has size one. 
	Consequently, each constructible argument is based on one assumption only (in contrast to e.g.\ $A_3$ in Example~\ref{ex:intro inst} relying on $a$ and $b$). 
\end{example}
%
%%% TO BE ADDED LATER -- ADDED.... 
Note that Bipolar ABA~\cite{DBLP:conf/prima/Cyras0T17} is a special kind of non-flat ABA restricted to being both atomic and additive.

As suggested by the previous example, each argument in an additive ABAF has the form $\{\}\vdash p$ or $\{a\}\vdash p$ for some assumption $a$. This induces the following bound. 
\begin{restatable}{proposition}{propArgNumberAdditive}
	\label{prop:arg number additive}
	The non-redundant core of an additive ABAF 
	$D$ 
	%$D=(\mathcal{L},\mathcal{R},\mathcal{A},\contraryempty)$ 
	consists of at most $(|\mathcal A| +1)\cdot |\mathcal L|$ many arguments.
\end{restatable}
Also in this case, Proposition~\ref{prop:complexity poly args} is applicable, including the computational benefits of lower complexity.

\section{Algorithms for Non-flat ABA via BAFs}
\label{sec:ababaf}
We introduce an approach for solving reasoning problems in non-flat ABA by instantiating a BAF and solving the corresponding reasoning problem in the BAF.
%\footnote{
	We focus on credulous reasoning for complete (and so preferred$'$) and stable semantics.
	%}
We employ efficient declarative methods, namely answer set programming (ASP)~\cite{GelfondL88,Niemela99} for generating arguments, and Boolean satisfiability (SAT)~\cite{DBLP:series/faia/336} %solving
for BAF reasoning. %(see background in the %appendix
%supplementary material). %\todo{TODO: SAT and ASP preliminaries to appendix}

\subsection{BAF Generation}
\label{sec:baf-gen}
% Given that an ABA framework gives rise to an exponential number of arguments in the worst case, it is essential to construct arguments as efficiently as possible.
We introduce a novel approach for generating non-redundant arguments (see Section~\ref{sec:feasible}) from an ABAF, using the state-of-the-art ASP solver \clingo~\cite{GebserKKOSW16}.% and making use of the redundancy notions introduced in Section~\ref{sec:feasible}.

Firstly, we present an ABAF $D=(\mathcal{L},\mathcal{R},\mathcal{A},\contraryempty)$ with $\mathcal{R} = \{r_1,\ldots,r_n\}$ in ASP as follows.
\begin{align*}
	\mathtt{ABA}(D) =
	&\{ \aspassump(a). \mid a \in \mathcal{A}  \}\ \cup\ \\
	&\{ \asphead(i,b). \mid r_i \in \mathcal{R}, b = head(r_i)  \}\ \cup\\
	&\{ \aspbody(i,b). \mid r_i \in \mathcal{R}, b \in body(r_i)  \}\ \cup \\
	&\{ \aspcontrary(a,b). \mid b = \overline{a}, a \in \mathcal{A}  \}.
\end{align*}
% In words, $\aspassump(a)$ indicates that $a$ is an assumption and $\aspcontrary(a,b)$ that $b$ is the contrary of $a$.
% %The rules are assumed to be indexed and are expressed via separate predicates for heads and bodies of rules.
% The predicate $\asphead(i,b)$ indicates that $b$ is the head of the 
% $i$th rule, while $\aspbody(i,b)$ indicates that $b$ is in the body of the $i$th rule.

Then, we introduce the ASP
%    \begin{align*}
	%    \{\aspin(X) &: \aspassump(X)\}. \\
	%\aspderivable(X) &\la \aspassump(X), \aspin(X).\\
	%\aspderivable(X) &\la \asphead(R,X), {\bf usable\_by\_in(R)}.\\
	%{\bf usable\_by\_in(R)} &\la \asphead(R,_), \aspderivable(X) : \aspbody(R,X).
	%    \end{align*}
%program $\aspmodule{generate\_{argument}}$ in Listing~\ref{asp:baf}, 
program $\aspmodule{gen\_arg}$ (see Listing~\ref{asp:baf}), which by itself enumerates each assumption set and determines what can be derived from it.
The following procedure limits redundancy in the set of arguments.\footnote{We chose to include assumption redundant arguments still, leaving their elimination to future work. }
Firstly, by Proposition~\ref{prop:redundant args}, %it is enough to construct $\subseteq$-minimal assumption sets deriving each atom (eliminating derivation redundancy).
we eliminate derivation redundant arguments.
We use \clingo{} heuristics~\cite{DBLP:conf/aaai/GebserKROSW13}
for this, adding the heuristic that each assumption is by default not in.
% by adding for each assumption $a\in\mathcal{A}$ the directive \texttt{\#heuristic in(a) [1, false]} to $\aspmodule{gen\_arg}$.
Secondly, %arguments need to only be constructed for a certain subset of $\mathcal{L}$. %, say \emph{relevant} arguments. \todo{this and following read a bit confusing: maybe add conclusions of arguments and $\Pi_{arguments}$ is unclear.}
% By Proposition~\ref{prop:expendable args}, it is enough to construct arguments for assumptions and atoms that are contrary to some assumption while retaining semantic equivalence; assumptions and contrary atoms are thus relevant.
by Proposition~\ref{prop:expendable args}, one retains semantic equivalence even if expendable arguments are %constructed only for assumptions and atoms that are contrary to an assumption
ignored.
As we are interested in credulous %and skeptical)
acceptance, we also construct arguments for atoms whose acceptance we want to query.
Optionally, instead of a query atom, a set of atoms can be specified, so that the acceptance of any of these atoms can be decided from a single BAF.
We enumerate answers to $\mathtt{ABA}(D) \cup \aspmodule{arg\_gen}$ together with a constraint that  %with \clingo{} with %the solving assumption
{\bf derivable}($a$) holds, for each atom of interest $a$, resulting in answer sets corresponding to %$\subseteq$-minimal core 
non-derivation redundant, non-expendable arguments for $a$.
%To enable the use of domain heuristics and the enumeration of all $\subseteq$-minimal answers, we use the \clingo{} arguments \texttt{--models=0, --heuristic=domain, --enum-mode=domRec}.
The attack and support relations over the arguments can be straightforwardly determined based on the assumptions and conclusion of each argument.
% \todo{TODO: add citation (and possibly explanation, prob in appendx) for clingo heuristics}

\begin{listing}[t]
	\caption{Program $\aspmodule{gen\_{arg}}$\label{asp:baf}}
	\begin{lstlisting}
		{in(X) : assumption(X)}.
		derivable(X) :- assumption(X), in(X).
		derivable(X) :- head(R,X), usable_rule(R).
		usable_rule(R) :- head(R,_), derivable(X) : body(R,X).
	\end{lstlisting}
\end{listing}

\subsection{SAT Encodings for BAFs}
\label{sec:baf-sat}
We move on to proposing SAT encodings that capture the BAF semantics of interest% (complete and stable; encoding for admissible semantics can be found in the appendix)
, following Definition~\ref{def:extsem}.
%    We introduce SAT encodings for finding complete and stable extensions and credulous acceptance under those semantics; an encoding for admissible semantics can be found in the appendix.
Our encodings are based on the standard SAT encodings for AFs~\cite{DBLP:conf/nmr/BesnardD04} with modifications to capture the additional %conditions 
aspects of BAFs.
Given a BAF $\CF = (\args, \attacker, \supporter)$ and the closure $cl(a)$ of each argument $a\in A$  (which can be precomputed in polynomial time), we capture conflict-free, closed, self-defending, and complete sets of arguments, respectively, as follows.
We use variables $x_a$ for $a\in A$ to denote that $a$ is in the extension.

%\begin{align*}
%    cf(F) =& \bigwedge_{(a,b)\in Att} (\neg x_a \vee \neg x_b) 
%\end{align*}
%\begin{align*}
%    closed(F) =& \bigwedge_{(a,b)\in Sup} (\neg x_a\vee x_b)
%\end{align*}
%\begin{align*}
%    com(F) =& cf(F)\wedge closed(F)\wedge self\_defense(F) \\ &\wedge  \bigwedge_{b\in A} \left( \left( \bigwedge_{(a,b)\in Att} \bigvee_{\substack{(c,d)\in Att,\\ d\in cl(a)}} x_c \right) \to x_b \right)
%\end{align*}

Conflict-freeness is encoded as for AFs: $cf(\CF) = \bigwedge_{(a,b)\in Att} (\neg x_a \vee \neg x_b)$. 
For closedness, if an argument $a$ is in an extension, then each argument $b$ that $a$ supports must be in the extension, too: $closed(\CF) = \bigwedge_{(a,b)\in Sup} (\neg x_a\vee x_b)$.

Note that a set of arguments $E$ defends an argument $b$ if and only if $E$ attacks the closure of every argument $a$ that attacks $b$~\cite[Lemma 3.4]{DBLP:conf/aaai/0001PRT24}.
Thus an extension $E$ defends itself iff it holds that, if $b\in E$, then for each $a$ that attacks $b$, some argument $d$ that attacks an argument $c$ in the closure of $a$ must be in the extension.
\begin{align*}
	\mathit{self}\_\mathit{defense}(\CF) = \bigwedge_{(a,b)\in Att} \left(x_b \to \bigvee_{\substack{(d,c)\in Att,\\ c\in cl(a)}} x_d\right)
\end{align*}

For a complete extension $E$, any argument $b$ must be in %the extension 
$E$ if it holds that an argument in the closure of each argument $a$ that attacks $b$ is attacked by an argument in the extension.
\begin{align*}
	\mathit{defended}(\CF) =  \bigwedge_{b\in A} \left( \left( \bigwedge_{(a,b)\in Att} \bigvee_{\substack{(d,c)\in Att,\\ c\in cl(a)}} x_d \right) \to x_b \right)
\end{align*}
Taken together, we encode complete semantics by $com(\CF) = cf(F)\wedge closed(\CF)\wedge \mathit{self}\_\mathit{defense}(\CF) \wedge \mathit{defended}(\CF)$.

We can encode stability similarly to the standard encoding for AFs, with only the addition that an extension is closed: 
%\begin{align*}
$$stb(F) = cf(\CF)\wedge closed(\CF)\wedge \bigwedge_{a\in A} \left( x_a \vee \bigvee_{(b,a)\in Att} x_b \right).$$
%\end{align*}

To decide credulous acceptance under $\sigma\in\{\com,\stb\}$%(and thus $\prf$)
, we add the clause $$cred(\CF,\alpha)=\bigvee_{a\in A_{\alpha}} x_a$$ where $A_\alpha\subseteq\mathcal{A}$ is the set of arguments concluding $\alpha$.
% $cred(\CF,\alpha)=\bigvee_{\alpha\in asms(a)} x_a$.
% The answer is positive if $\sigma(\CF)\wedge cred(\CF,\alpha)$ is satisfiable.

\begin{proposition}
	Given an ABAF $D=(\mathcal{L},\mathcal{R},\mathcal{A},\contraryempty)$,  $\sigma\in\{\com,\stb\}$ and $\alpha\in\mathcal{L}$, let $\CF$ be the BAF constructed from $D$ via the procedure of Section~\ref{sec:baf-gen}.
	Then $\sigma(\CF)\wedge cred(\CF,\alpha)$ is satisfiable iff $\alpha$ is credulously accepted under~$\sigma$.%
\end{proposition}

\section{ASP Algorithms for Non-flat ABA}
\label{sec:cegar}

In this section we introduce an ASP-based approach for credulous acceptance in non-flat ABA without constructing arguments.
We propose a counterexample-guided abstraction refinement (CEGAR)~\cite{DBLP:journals/tcad/ClarkeGS04,DBLP:journals/jacm/ClarkeGJLV03} algorithm for complete semantics, inspired by state-of-the-art algorithms for other problems in structured argumentation, including flat ABA, that are hard for the second level of the polynomial hierarchy~\cite{LehtonenWJ21b,DBLP:conf/kr/LehtonenWJ22,DBLP:conf/comma/LehtonenWJ22}.
For stable semantics, we propose an ASP encoding% (adapting the standard encoding for flat ABA~\cite{LehtonenWJ21a})
, reflecting the fact that credulous acceptance is NP-complete under stable semantics.

% In a CEGAR algorithm, an over-approximation (abstraction) of the solution space is iteratively refined by drawing candidates from this space and verifying if the candidate is a solution.
% Candidate solutions are computed with an NP-solver and another solver call is made to check if there is a counterexample to the candidate being a solution.
% If there is no counterexample, the candidate is a solution to the original problem.
% Otherwise the abstraction is refined by analyzing the counterexample, and the search is continued.

We use ASP to generate candidates based on an abstraction of the original problem, and another ASP solver to verify whether a counterexample to the candidate being a solution exists. 
We use the ASP encoding of ABAFs %specified in 
from Section~\ref{sec:baf-gen}, and
% For convenience, we define the shorthand $\mathit{constr}(M) =\ \la l_1,\ldots,l_n$ where $M=\{l_1,\ldots,l_n\}$, i.e. $\mathit{constr}(M)$ enforces that not all members of $M$ hold.
%In our programs, we use the predicates $\aspin, \aspout, {\bf undefeated}$ and ${\bf %target}$ to 
the notation $solve(\aspmodule{})$ to refer to solving the ASP program $\aspmodule{}$ and assume that $solve(\aspmodule{})$ either returns an answer set or reports that the program is unsatisfiable.
Details of the ASP programs % and an algorithm for admissible semantics 
are available in~\cite{lehtonen2024instantiations}.
%are in the %appendix.
%supplementary material.

%requires a candidate to derive the queried atom (Line~\ref{alg:adm-cred}-1).

We introduce Algorithm~\ref{alg:com-cred} for credulous acceptance under complete semantics.
As subprocedures, we introduce three ASP programs: the abstraction $\aspmodule{abs}$, a program that checks for a counterexample to admissibility $\aspmodule{not\_adm}$, and $\aspmodule{defends}(C,a)$ for checking counterexamples to completeness, i.e., if particular assumptions are defended.
The abstraction $\aspmodule{abs}$ admits as answers closed and conflict-free assumption sets from which the query is derivable, along with a stronger condition for complete semantics.
Namely, if an assumption is not in the candidate and not attacked by the set of undefeated assumptions, the candidate cannot be complete (since this assumption cannot be attacked by any closed  assumption set in particular, and is thus defended by the candidate).
The program $\aspmodule{not\_adm}$ is satisfiable iff a closed set of assumptions that is not attacked by the candidate attacks the candidate, implying that the candidate is not admissible.
Finally, given a candidate $C\subseteq \mathcal{A}$ and an assumption $a\in\mathcal{A}$, the program $\aspmodule{defends}(C,a)$ is satisfiable iff there is a set of assumptions that is not attacked by $C$, is closed, and attacks $a$.
In such a case the candidate does not defend $a$.

Algorithm~\ref{alg:com-cred} iteratively generates candidates (Line~\ref{alg:com-cred-1}) and checks admissibility \linebreak (Line~\ref{alg:com-cred-3}).
In case a candidate is admissible, it is further checked whether there is an assumption that is not in the candidate or attacked by the candidate (Line~\ref{alg:com-cred-4}), such that this assumption is defended by the candidate (Line~\ref{alg:com-cred-5}).
If not, the candidate is complete and thus the queried atom is credulously accepted (Line~\ref{alg:com-cred-6}).
Otherwise (if $C$ is either not admissible or not complete) the abstraction is refined by excluding the candidate from further consideration (Line~\ref{alg:com-cred-7}).
Finally, if no complete assumption set is found, the queried atom is not credulously accepted (Line~\ref{alg:com-cred-8}).

\begin{algorithm}[t]
	\caption{Credulous acceptance, complete semantics}
	\label{alg:com-cred}
	\begin{algorithmic}[1]
		\REQUIRE ABA framework $F=(\mathcal{L},\mathcal{R},\mathcal{A},\contraryempty)$, $s\in \mathcal{L}$
		\ENSURE return YES if $s$ is credulously accepted under complete semantics in $F$, NO otherwise
		\STATE{$\algorithmicwhile\ C:= solve(\aspmodule{abs})\ \algorithmicdo$} \label{alg:com-cred-1}
		%\STATE{\hspace{\algorithmicindent}Let $C$ be a candidate}
		\STATE{\hspace{\algorithmicindent}\textbf{$\mathit{flag}:=T$}} \label{alg:com-cred-2}
		\STATE{\hspace{\algorithmicindent}$\algorithmicif\ solve(\aspmodule{not\_adm}(C))$ unsatisfiable $ \algorithmicthen$} \label{alg:com-cred-3}
		\STATE{\hspace{\algorithmicindent}\hspace{\algorithmicindent}$\algorithmicfor$ $a\in \mathcal{A}$ s.t. $a\notin C$ and $a$ not attacked by $C$\ $\algorithmicdo$} \label{alg:com-cred-4}
		\STATE{\hspace{\algorithmicindent}\hspace{\algorithmicindent}\hspace{\algorithmicindent} $\algorithmicif\ solve(\aspmodule{defends}(C,a))\ \algorithmicthen\ \mathit{flag}:=F$; break} \label{alg:com-cred-5}
		\STATE{\hspace{\algorithmicindent}\hspace{\algorithmicindent}$\algorithmicif\ \mathit{flag}=T\ \algorithmicthen\ \algorithmicreturn$ YES} \label{alg:com-cred-6}
		\STATE{\hspace{\algorithmicindent}Add constraint excluding $C$ to $\aspmodule{abs}$} \label{alg:com-cred-7}
		\RETURN{NO} \label{alg:com-cred-8}
	\end{algorithmic}
\end{algorithm}

	\begin{proposition}
		Given an ABAF $D=(\mathcal{L},\mathcal{R},\mathcal{A},\contraryempty)$ and a query $\alpha\in\mathcal{L}$, Algorithm~\ref{alg:com-cred} outputs YES if and only if $\alpha$ is credulously accepted under complete semantics in $D$.
	\end{proposition}
	
	% We introduce an ASP encoding for non-flat 
	Stable semantics can be encoded 
	% in ASP
	by slightly adapting the encoding for flat ABAFs \cite{LehtonenWJ21a}. 
	% Compared to the encoding for flat ABA%~\cite{LehtonenWJ21a}
	Namely, we add a constraint for the assumption set being closed.
	%the constraint $\la \aspassump(X),\aspderivable(X),\naf \aspin(X)$, prohibiting any assumption not in the stable assumption set from being derivable by it. 
	For finding the credulous acceptance of %a query 
	$s\in\mathcal{L}$, %it suffices to
	we add a constraint requiring that $s$ is derivable from the assumption set. 
	% For skeptical acceptance, again the constraint is reverted to requiring that the query is not derivable, and if the program is satisfiable, the query is not skeptically accepted.
	%$\la \naf {\bf derivable}(s)$; $s$ is accepted if there is an answer to the program.
	% For skeptical acceptance, again the constraint is reverted to requiring that the query is not derivable, and if the program is satisfiable, the query is not skeptically accepted.
	%the constraint $\la{\bf derivable}(s)$ should be added instead; $s$ is accepted if there is no answer to the program.
	
	\section{Empirical Evaluation}
	\label{sec:experiments}
	
	We present an evaluation of the algorithms proposed in Sections~\ref{sec:ababaf} and \ref{sec:cegar}, named \ababaf{}\footnote{Available at \url{https://bitbucket.org/lehtonen/ababaf}.} and \cegar{}\footnote{Available at \url{https://bitbucket.org/coreo-group/aspforaba}.}, respectively.
	% The implementations will be made available in open source.
	We implemented both approaches in Python, using \clingo{} (version 5.5.1)~\cite{GebserKKOSW16} for \cegar{} and for generating the arguments in \ababaf{}.
	We used \textsc{PySAT} (version 0.1.7)~\cite{imms-sat18} with \textsc{Glucose} (version 1.0.3)~\cite{DBLP:conf/ijcai/AudemardS09,DBLP:conf/sat/EenS03} as the SAT solver in \ababaf{}. % to solve the BAF reasoning problems.
	We used 2.50 GHz Intel Xeon Gold 6248 machines under a per-instance time limit of 600 seconds and memory limit of 32 GB.
	
	% \paragraph{Benchmarks} 
	Lacking a standard benchmark library for non-flat ABA, we generated two benchmark sets adapted from flat ABA benchmarks~\cite{DBLP:conf/kr/JarvisaloLN23}. Set 1 has the following parameters:
	% number of atoms in $\{80,120,160,200\}$, proportion of atoms that are assumptions in $\{0.2,0.4\}$, proportion of assumptions that have rules deriving them in $\{0.2,0.5\}$, and the number of rules that have a given atom as a head as well as rule size (number of atoms in the body of rule), was selected uniformly at random from the interval $[1,n]$ for $n\in\{1,2,5\}$.
	number of atoms in $\{80,120,160,200\}$, ratio of atoms that are assumptions in $\{0.2,0.4\}$, ratio of assumptions occurring as rule heads in $\{0.2,0.5\}$, and both number of rules deriving any given atom and rule size (number of atoms in the body of rule) selected at random from the interval $[1,n]$ for $n\in\{1,2,5\}$.
	We call the maximum rules per atom $mr$ and maximum rule size $ms$; instances with $ms=1$ are additive.
	%We generated both non-circular and circular instances (the former by selecting a random permutation of the atoms and when generating a rule for an atom $a$, only allowing atoms earlier in the permutation than $a$ to be included in the rule body).
	% The first benchmarks set contains additive ABAFs, namely the instances with $ms=1$.
	% We generated 5 instances per each combination of parameters. %, for a total of 720 instances.
	For benchmark set 2, we limited $mr$ and $ms$ to $\{2,5\}$, and generated instances a certain distance from atomic.
	For this, a \emph{slack} parameter specifies how many atoms in each rule body can be non-assumptions.
	Here slack is $0,1$ or $2$, the first resulting in atomic ABAFs.
	We generated 5 instances for each combination of parameters for both benchmark sets. %, for a total of 960 instances.
	
	%\TODO{TODO appendix maybe: nonflat prob affects: \ababaf{} better with 0.2, \cegar{} slightly better with 0.5, assumption ratio: both, especially \cegar{} quite a bit worse with higher}
	
	% \paragraph{Results}
	Table~\ref{table:dc-co} summarizes results for credulous acceptance under complete semantics. %, with further details in the appendix.
	%Table~\ref{table:atomic-dc-co} shows results on credulous acceptance under complete semantics on random atomic instances.
	%The chosen parameters evidently matter for both the relative and absolute performance of both solvers.
	On benchmark set 1 (Table~\ref{table:dc-co}, left), generally \ababaf{} performs better than \cegar{} when the parameters take lower values.
	\ababaf{} outperforms \cegar{} when $ms=1$, corresponding to additive ABA frameworks, but also for $ms=2$ when $mr$ is 1 or 2.
	\cegar{}, on the other hand, performs best with $mr=5$.
	The results for benchmark set 2 (for credulous acceptance under complete semantics) can be seen in Table~\ref{table:dc-co} (right).
	On atomic instances ($slack=0$) \ababaf{} outperforms \cegar{}.
	As the slack increases, the performance of \ababaf{} decreases, while \cegar{} somewhat improves.
	Evidently, \ababaf{} is able to take advantage of the lower complexity of additive and atomic instances.
	The results suggest that the approaches are complementary and their relative performance varies by parameters.
	
	%It seems likely that the size of the BAF that a given ABA framework gives rise to is significant for the performance of \ababaf{}, since \ababaf{} constructs the arguments in the first place, and the resulting BAF is input for the SAT call of \ababaf{}.
	We show the run times of both algorithms against the number of arguments a given instance gives rise to in Figure~\ref{fig:narguments} (for instances solved by\ababaf{}; timeouts of \cegar{} shown as 600 seconds).
	For \ababaf{}, there is a clear correlation between run time and size of the constructed BAF, as can be expected given that the arguments construction is time-consuming and the BAF is the input for a SAT call.
	In contrast, BAF size does not predict run time for \cegar{}, since 
	%reflecting the fact that 
	\cegar{} does not instantiate arguments.
	
	%\TODO{TODO: check mean number of arguments per parameter family! say something about number of arguments wrt atomic and additive frameworks}
	
	We also evaluated our algorithms under stable semantics.
	%We expected \ababaf{} to perform worse compared to admissible and complete semantics, given that acceptance under stable semantics is on the first level of the polynomial hierarchy and thus the \cegar{} can solve it with a single ASP call while \ababaf{} still requires the time-consuming argument construction process.
	As expected given the lower complexity compared to complete semantics, \cegar{} performs better, solving all instances.
	On the other hand, \ababaf{} constructs the BAF for stable semantics too, and accordingly the performance of \ababaf{} is significantly worse than \cegar{}, %and indeed comparable (but slightly better) to its performance for complete semantics,
	with multiple timeouts.
	This highlights that a crucial component in the performance of \ababaf{} on complete semantics is the lower complexity of deciding acceptance in BAF compared to the corresponding ABAF.
	
	\begin{table}[t]
		%\vspace{-2mm}
		\centering
		\resizebox{\linewidth}{!} {
			\small
			\begin{tabular}{cc|rr|rr}
				\toprule
				\multicolumn{2}{c}{} & \multicolumn{4}{c}{\textbf{\#solved} (mean run time (s))}\\
				%\hline
				\midrule
				$ms$ &
				$mr$ &
				\multicolumn{2}{c|}{\ababaf{}} &
				\multicolumn{2}{c}{\cegar{}} \\
				%\hline
				\midrule
				& 1 & 80 & (0.2) & 56 & (9.3) \\
				1 & 2 & 80 & (0.4) & 66 & (8.4) \\
				& 5 & 80 & (6.2) & 79 & (0.1) \\
				\midrule
				& 1 & 70 & (0.6) & 62 & (6.8) \\
				2 & 2 & 54 & (17.1) & 40 & (24.2) \\
				& 5 & 50 & (13.2) & 66 & (14.4) \\
				\midrule
				& 1 & 80 & (2.3) & 80 & (0.1) \\
				5 & 2 & 54 & (4.0) & 72 & (5.3) \\
				& 5 & 11 & (24.9) & 43 & (37.1) \\
				\bottomrule 
			\end{tabular}
			\quad
			\begin{tabular}{cc|rr|rr}
				\toprule
				\multicolumn{2}{c}{} & \multicolumn{4}{c}{\textbf{\#solved} (mean run time (s))}\\
				%\hline
				\midrule
				$slack$ &
				$ms$ &
				\multicolumn{2}{c}{\ababaf{}} &
				\multicolumn{2}{c}{\cegar{}} \\
				\midrule
				0 & 2 & 147 & (3.8) & 104 & (41.5) \\
				& 5 & 109 & (7.3) & 77 & (51.3) \\
				\midrule
				1 & 2 & 112 & (15.4) & 112 & (18.1) \\
				& 5 & 51 & (20.0) & 89 & (31.3) \\
				\midrule
				2 & 2 & 122 & (21.0) & 118 & (8.2) \\
				& 5 & 59 & (6.6) & 102 & (31.2) \\
				\bottomrule
			\end{tabular}
		}
		\caption{Number of solved instances and mean run time over solved instances under complete semantics in benchmark sets 1 (left) and 2 (right). There are 80 and 160 instances per row in sets 1 and 2.
			\label{table:dc-co}}
	\end{table}

			\begin{figure}[t]
				\centering
				\includegraphics[width=0.65\linewidth]{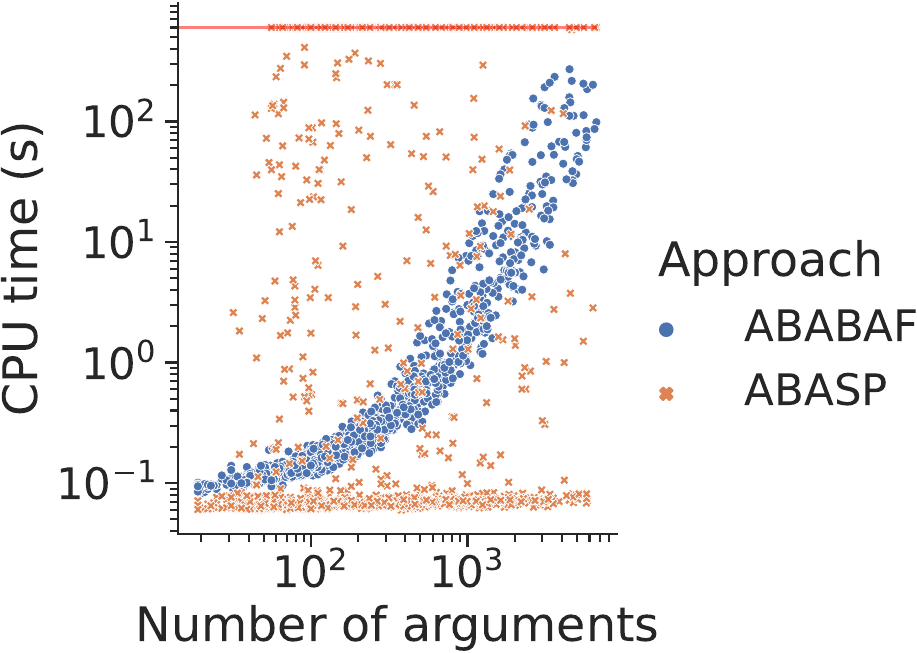}
				\caption{\label{fig:narguments} The effect of the number of arguments an ABAF gives rise to on the run time of our approaches w.r.t.\ complete semantics. %Timeouts are included as a running time of 600 seconds.
				}
			\end{figure}

			\section{Conclusion}
			Efficient algorithmic solutions for assumption-based argumentation are widely studied; however, despite the wide variety of different solvers for assumption-based reasoning, most existing approaches only focus on the flat ABA fragment. 
			In this work we investigated theoretical foundations and algorithms for non-flat ABA reasoning, which has higher complexity than in flat ABA. 
			
			Our redundancy notions gave rise to two fragments of ABA, i.e., atomic and additive ABA, in which reasoning exhibits the same complexity as for flat ABA.
			We proposed two algorithmic approaches, one instantiating a BAF and one without argument construction. 
			The former faces an exponential number of arguments in general, but we proposed and applied redundancy notions, and employed state-of-the-art solving techniques.
			For the latter, we adapted state-of-the-art algorithms for structured argumentation.
			We showed empirically that, in contrast to previous instantiation-based approaches for structured argumentation, our novel instantiation-based approach is able to outperform the direct approach, in particular for problems on the second level of the polynomial hierarchy. %, depending on the family of ABAFs to be solved. 
			
			For future work, we want to extend our implementation to further common semantics and reasoning tasks such as skeptical acceptance% and preferred semantics
			; this is an additional challenge, since e.g. skeptical reasoning under preferred semantics lies on the third level of the polynomial hierarchy. 
			Another promising line of future work is to develop methods for computing explanations for ABA in spirit of the work by~\citeauthor{dung2006dialectic}~(\citeyear{dung2006dialectic}), utilizing our instantiation-based implementation. 

\section*{Acknowledgements} 
This research was funded %in whole, or in part, 
by the Austrian Science Fund (FWF) P35632 and University of Helsinki Doctoral Programme in Computer Science DoCS;
by the  European Research Council (ERC) under the European Union’s Horizon 2020 research and innovation programme (grant agreement No. 101020934, ADIX) and by J.P. Morgan and by the Royal Academy of Engineering under the Research Chairs and Senior Research Fellowships scheme; 
by  the  Federal  Ministry  of  Education  and  Research  of  Germany  and  by  S\"achsische Staatsministerium  f\"ur  Wissenschaft,  Kultur  und  Tourismus  in  the  programme  Center  of Excellence for AI-research ``Center for Scalable Data Analytics and Artificial Intelligence Dresden/Leipzig'', project identification number:  ScaDS.AI. 
The authors wish to thank the Finnish Computing Competence Infrastructure (FCCI)
for supporting this project with computational and data storage resources.

\clearpage

\appendix

\section{Omitted Proofs of Section~\ref{subsec:lower bound}}

Formally, we show a hardness result for the class \compclass{coNP}. This class contains problems composed of pairs. Here, we consider the problem of pairs $(D,A)$, with $D$ a non-flat ABAF and $A$ an assumption set in $D$, that asks whether $A$ is admissible or complete in $D$. We show that this problem is \compclass{coNP}-hard, signaling that $D$ cannot be transformed to a structure of polynomial size from which one can in polynomial time check admissibility of $A$ in $D$. % (unless the polynomial hierarchy collapses). 

\begin{restatable}{theorem}{thLowerBound}
	\label{th:lower bound}
	The problem of deciding whether a given set of assumptions is admissible or complete in a given non-flat ABA framework is \compclass{coNP}-hard under \compred{} reductions.
\end{restatable}
\begin{proof}
	The problem of clausal inference (CI) has as instances pairs $(\phi,c)$ with $\phi$ a Boolean formula in conjunctive normal form (CNF) and $c$ a clause. The task is to decide whether $\phi \models c$. 
	This problem is \compclass{coNP}-complete~\cite[Theorem 2.10]{CadoliDLS02}.
	We make use of \compred reductions, which for our proof can be simplified (i.e., we do not make use of the full expressivity of \compred reductions). 
	Given a pair $(X,Y)$ a (simplified) \compred reduction is composed of functions $f$ and $g$, the former is a poly-size function (i.e., returning a structure polynomially bounded by its input) and $g$ is a polynomial-time function. An instance $(X,Y)$ of problem $A$ \compred{}-reduces to $B$ if it holds that $(f(X),g(Y))$ is a ``yes'' instance of $B$ iff $(X,Y)$ is a ``yes'' instance of $A$. 
	% 	The problem of deciding whether given two components $\phi$ and $c$, with $\phi$ a Boolean formula in CNF and $c$ a clause, to decide whether $\phi \models c$ is $||\!\!\!\leadsto$coNP-complete~\cite[Theorem 2.10]{CadoliDLS02}. 
	
	We \compred{}-reduce from this problem as follows with
	$C = \{c_1,\ldots,c_n\}$ the set of clauses of $\phi$, $X$ the set of variables of $\phi$, $X' = \{x' \mid x \in X\}$, $Y = \{y_x \mid x \in X\}$, $\neg Y = \{\neg y_x \mid y_x \in Y\}$, and $\neg P = \{\neg p \mid p \in P\}$. 
	In this proof, we use contraries for each $p \in X \cup \neg X \cup Y \cup \neg Y \cup X' \cup \neg X'$ by simply denoting it as an atom $\overline{p}$.
	\begin{align*}
		\mathcal{A} =\ & X \cup X' \cup \neg X \cup \neg X' \cup Y \cup \neg Y\cup \{d,w\}\\
		\mathcal{R} =\ & \{\overline{d} \leftarrow w\}\\
		&\{d \leftarrow x', \neg x' \mid x' \in X'\}\\
		&\{c \leftarrow l' \mid l\in c, c \in \phi\}\\
		&\{\overline{p} \leftarrow c_1,\ldots,c_n \mid p \in X \cup \neg X \cup \{w\}\}\\
		& \{\overline{y_x} \leftarrow x\} \cup \{\overline{\neg y_x} \leftarrow \neg x\} \\
		& \{\overline{x} \leftarrow y_x\} \cup \{\overline{\neg x} \leftarrow \neg y_x\} \\
		&\{\overline{x'} \leftarrow x \mid x \in X\} \cup \{\overline{\neg x'} \leftarrow \neg x \mid x \in X\}\\
		&\{\overline{x'} \leftarrow x' \mid x \in X\} \cup \{\overline{\neg x'} \leftarrow \neg x' \mid x \in X\}
		%\textnormal{contraries}\ % & \overline{x} = y_x, \forall x \in X, \overline{\neg x} = \neg y_x, \forall \neg x \in \neg X\\
		%& \overline{y_x} = x, \forall y_x \in Y, \overline{\neg y_x} = \neg x, \forall \neg y_x \in \neg Y\\
		%& \overline{x'} = x, \overline{\neg x'} = \neg x, \forall x' \in X
	\end{align*}
	That is, $f$ takes as input the formula $\phi$ and returns $D = (\mathcal{L},\mathcal{R},\mathcal{A},\contraryempty)$. We define $g(c) = \{z \in c\} \cup \{w\}$ (i.e., the set of literals in clause $c$ and the fresh "$w$").
	We claim that $g(c)$ is admissible in the constructed ABA framework iff $\phi \models c$. It is direct that $f$ is a poly-sized function and $g$ is poly-time (in fact, both are poly-time). 
	
	Assume that $\phi \models c$. Suppose that $g(c)$ is not admissible in $D$. First note that $g(c)$ is conflict-free in $D$. For $g(c)$ not be admissible it must be the case that $g(c)$ does not defend itself. Then there is a (closed) set of assumptions $A$ such that $A$ attacks $g(c)$ and $g(c)$ does not attack $A$. By construction, if $A$ contains some $y_x$ or $\neg y_x$, then $A$ is counter-attacked by $g(c)$. Thus, by construction, $A$ derives some contrary of $g(c)$ via one of the rules of form $\overline{p} \leftarrow c_1,\ldots,c_n$ and $A$ does not contain (or derive) $d$. If $A$ contains $d$, then any set containing $w$ (like $g(c)$) attacks $A$. If $A$ derives $d$ but does not contain $d$, then $A$ is not closed. This implies that for any $x\in X$ we find that $\{x',\neg x'\}\nsubseteq A$, i.e., $A$ defines a partial truth-value assignment $\tau$ on $X$. By construction, $\tau$ is a model of $\phi$. By presumption that $\phi \models c$, we find that $\tau$ satisfies at least one literal in $c$ ($\tau$ is partial, but this partial assignment suffices to show satisfaction of $\phi$, thus any completion satisfies $\phi$). If $c$ is the empty clause, then $g(c)$ attacks no primed assumption set representing a partial truth-value assignment (if $\phi \models c$ then $\phi$ is unsatisfiable and no partial truth-value assignment satisfies $\phi$). It holds that $g(c)$ attacks $A$, since there is some primed literal in $A$ and the corresponding unprimed literal is in $g(c)$, a contradiction. Thus, $g(c)$ is admissible in $D$. To see that $g(c)$ is complete, consider any assumption outside $g(c)$. Any assumption in $X'$ and $\neg X'$ is self-contradictory. Assumption $d$ is attacked by $g(c)$, as well as the complementary literals in $Y$ and $\neg Y$. The remaining ones are attacked by unattacked assumptions.
	
	For the other direction, assume that $g(c)$ is complete in $D$. Suppose that $\phi \not \models c$. Then there is a truth-value assignment $\tau \models \phi$ and $\tau \not \models c$. Construct set of assumptions $A$ that is composed of all primed literals satisfied by $\tau$. Since $\tau \models \phi$, we find that all $c_i$ are derived from $A$, and, in turn, also $s$. It holds that $d$ is not derived by $A$. Thus, $A$ is closed and attacks any assumption set containing at least one literal. Thus, $A$ attacks $g(c)$. By presumption of $g(c)$ being admissible, it follows that $g(c)$ attacks $A$. But then there is an unprimed literal in $g(c)$ that is in $A$. But then $\tau$ satisfies on literal in $c$, and, in turn $\tau \models c$, a contradiction. It follows that $\phi \models c$. The same line of reasoning applies when assuming that $g(c)$ is admissible in $D$.
	% 	
	% 	Idea: the primed variables represent interpretation over the Boolean variables. If it's not well-defined (both $x'$ to true and false), then the set is not closed (derives $d$). Anything containing $d$ is attacked by any assumption set of the non-primed variables. Any model attacks the non-primed variables. However, they attack back if one variable is assigned the same value as the literal in the clause. Then the idea is that one needs to have all models assign at least one literal in $c$ to true.
\end{proof}

From this previous result and observations, the following theorem is direct to show, using earlier results~\cite{CadoliDLS02}. In particular, if one can transform a given non-flat ABA framework to some structure from which one can infer that a given set of assumptions is admissible (complete), then the same holds for clausal inference (CI, see proof above). Then, by earlier results~\cite{CadoliDLS02}, the polynomial hierarchy collapses.

\thLowerBoundBAF*

\section{Omitted Proofs of Section~\ref{sec:feasible}}

Below, we let $E^\oplus_F = E\cup E^+_F$ where $E^+_F = \{ a\in A \mid E\text{ attacks }a \}$ denote the \textit{range} of a set $E$.

Before proving Proposition~\ref{prop:redundant args}, let us first note that the closure function within any BAF is \emph{additive}, \ie we have %if 
$\cl(E) = \bigcup_{e\in E}\cl(\{e\})$
for any set $E$ of arguments. 
%Note that this is in general not the case of a non-flat ABAF $D$ (but this constitutes an interesting fragment as we discuss in Section~\ref{sec:additive} below). 
%Hence, the closure function $\cl$ becomes simpler when instantiating the BAF $\CF_D$.
%Recall that we call $\cl(E) = \bigcup_{n\geq 1} \mu^n(E)$ the \emph{closure} of $E$, where $\mu(E) = E\cup \{ a\in A \mid \exists e\in E:\; (e,a)\in\supporter \}$.  
\begin{restatable}{lemma}{leAdditiveClosure}
	\label{le:additive closure}
	Let $\CF =(\args,\attacker,\supporter)$ be a BAF. 
	Then for any set $E\subseteq A$ of arguments we have $\cl(E) = \bigcup_{e\in E}\cl(\{e\})$. 
\end{restatable} 
\begin{proof} 
	This follows directly from the existential quantifier in the definition of the closure operator (recall that $E$ is closed iff $E= E\cup \{ a\in A \mid \exists e\in E:\; (e,a)\in\supporter \}$).
\end{proof}

\propRedundantArgs*
\begin{proof}
	Let $x = (S\vdash p)$ and $y = (S'\vdash p)$ with $S'\subsetneq S$; such $y$ must exist because $x$ is redundant. 
	
	($\subseteq$)
	Take some $E\in\sigma(\CF_D)$. We construct a corresponding extension in $\CG$. 
	
	(Case 1: $x\notin E$) 
	In this case, $E\in\sigma(\CG)$ follows directly.	
 
	(Case 2: $x\in E$)
	%\begin{itemize}
		%\item 
		%($\sigma\in \{ \com,\grd,\stb \}$) 
		%If $x\in E$, then for each $a\in\asms(E)$, the canonical argument $\{a\}\vdash s$ is defended by $E$. 
        %
		Due to $\asms(y)\subseteq \asms(x)$, $y\in E$ follows. 
		Now set $E_x = E\setminus \{x\}$. In particular, $y\in E_x$.
		Hence, $E^\oplus_{\CF_D} = (E_x)^\oplus_\CG \cup \{x\}$, \ie the range of $E_x$ in $\CG$ is the same as the range of $E$ in $\CF$ except that $x$ is not present anymore. 
		Moreover, $E_x$ (in $\CG$) and $E$ (in $\CF$) have the same (sources for) in-coming attacks. 
		Consequently, if $\{a\}\vdash a$ is defended by $E$ in $\CF$, then the same is true for $E_x$ in $\CG$. 
		Hence $\asms(E) = \asms(E_x)$. 
		Thus the following properties are immediate: 
		\begin{itemize}
			\item from $E\in\cf(\CF)$, it follows that $E_x\in\cf(\CG)$;
			\item $\cl(E) = \bigcup_{e\in E}\cl(\{e\})
			= \bigcup_{e\in E_x}\cl(\{e\}) = \cl(E_x)$;
			\item $\Gamma_{\CF_D}(E) = \Gamma_\CG(E_x)\cup \{x\}$;
			\item $E^\oplus_{\CF_D} = (E_x)^\oplus_\CG \cup \{x\}$ 
		\end{itemize}
		which suffices to show that $E_x\in\sigma(\CG)$. 
		
		%\item 
		%($\sigma\in \{ \adm,\prf \}$) 
		%Due to $\asms(y)\subseteq \asms(x)$, $E$ defends the set $\{ (\{a\} \vdash a) \mid a\in\asms(x) \}$ of assumption arguments and $E$ defends $y$. 
		%Set $E' = E\cup \{y\} \cup \{ (\{a\} \vdash a) \mid a\in\asms(x) \}$. 
		%Then, 
		%$E'$ defends itself and due to 
		%$\asms(E) = \asms(E')$ we have
		%$\cl(E) = \cl(E')$ by Lemma~\ref{le:additive closure};
		%hence $E'$ is closed. 
		%We reason as for the above case to see that $E'\setminus \{x\}$ must be admissible (resp. preferred) in $\CG$. 
	%\end{itemize}
	
	($\supseteq$)
	Take some $E\in\sigma(\CG)$. In this case, we have of course $x\notin E$. 
	
	(Case 1: $S\nsubseteq\asms(E)$) We show that $E\in\sigma(\CF_D)$. 
	
	\begin{itemize}
		\item 
		(conflict-free)
		It is clear that $E$ is conflict-free in $\CF_D$. 
		
		\item 
		(defense) 
		If $E$ is attacked by $\cl(\{x\})$ in $\CF_D$, then it is also attacked by $\cl(\{y\})$ in $\CF_D$. Hence it is attacked by $\cl(\{y\})$ in $\CG$. It thus defends itself against $y$ due to being admissible. Consequently, it defends itself against $\cl(\{x\})$. 
		Defense against any other argument is clear. 
		
		\item 
		(closed)
		This is clear. 
		
		\item 
		(fixed point) 
		Since $S\nsubseteq\asms(E)$, there is at least one assumption $a\in S$ not defended by $E$. 
		Hence $E$ does not defend $x$ in $\CF_D$. 
	\end{itemize}
	
	(Case 2: $S\subseteq\asms(E)$). By similar reasoning, $E\cup \{x\}\in\sigma(\CF_D)$. 
\end{proof}

\propExpendableArgs*
\begin{proof}%[Sketch]
	We reason similar as in the proof of Proposition~\ref{prop:redundant args}. 
	Let $x = (S\vdash p)$ be an expendable argument. 
	
    ($\subseteq$) Let $E\in\sigma(\CF_D)$.
    
    (Case 1: $x\notin E$) 
	In this case, $E\in\sigma(\CG)$ follows directly.	
 
	(Case 2: $x\in E$)
    Let $E_x = E\setminus \{x\}$.
    By completeness, $\{ (\{a\} \vdash a ) \mid a \in S \}\subseteq E$. 
	   Note that $x$ does not attack another argument (since $x\notin \contrary{\mathcal{A}}$), hence $x$ does not influence the range of $E$.
        Therefore, the proof of Proposition~\ref{prop:redundant args} translates to this situation, i.e., 
		\begin{itemize}
			\item from $E\in\cf(\CF)$, it follows that $E_x\in\cf(\CG)$;
			\item $\cl(E) = \bigcup_{e\in E}\cl(\{e\})
			= \bigcup_{e\in E_x}\cl(\{e\}) = \cl(E_x)$;
			\item $\Gamma_{\CF_D}(E) = \Gamma_\CG(E_x)\cup \{x\}$;
			\item $E^\oplus_{\CF_D} = (E_x)^\oplus_\CG \cup \{x\}$ 
		\end{itemize}
		We obtain $E_x\in\sigma(\CG)$. 

     ($\supseteq$) Let $E\in\sigma(\CG)$. In this case, $x\notin E$.

     We reason similar as in the proof of Proposition~\ref{prop:redundant args} to show that, if $S\nsubseteq\asms(E)$, it holds that $E\in\sigma(\CF_D)$, and if $S\subseteq\asms(E)$, it holds that $E\cup \{x\}\in\sigma(\CF_D)$.
	%(Case 1: $\sigma\in \{ \com,\grd,\stb \}$) 
	%By completeness, $\{ (\{a\} \vdash a ) \mid a \in S \}\subseteq E$. 
	%Since $x$ does not influence the range of $E$, the proof of case 2 in Proposition~\ref{prop:redundant args} translates to this situation. 
	%
	%(Case 1: $\sigma\in \{ \adm,\prf \}$) 
	%By admissibility, all arguments in $\{ (\{a\} \vdash a ) \mid a \in S \}$ are defended by $E$. 
	%Hence the proof of case 2 in Proposition~\ref{prop:redundant args} translates to this situation. 
\end{proof}

\propAssumptionRedArgs*
\begin{proof}
	Let $x = (S\vdash p)\in A$ and let $x'=(S'\vdash a)$ be the proper sub-argument of $x$ with $S'\subseteq S$ and $a\in\mathcal{A}$, witnessing its assumption redundancy. Furthermore, let $y$ denote the argument which arises when removing $x'$ from $x$, i.e., $y=(S''\cup \{a\}\vdash p)$ with $S''\subseteq S$.
	
	($\subseteq$)
	Let $E\in\sigma(\CF_D)$. We show that $E'=E\setminus \{x\}\in\sigma(\CG)$.
	
	\begin{itemize}
		\item (conflict-free) This is clear.  
		\item (closed) $E'$ is closed in $\CG$:
		
		$E$ is closed in $\CF_D$, that is, $E$ contains all arguments it supports. Since $E'$ is a proper subset of $E$, it holds that if an argument $a$ is supported by $E'$ then it is also supported by $E$, i.e., 
		$\{e\mid (z,a)\in \supporter, z\in E\}\subseteq \{e\mid (z,a)\in \supporter, z\in E'\}$. 
		We obtain that $E'$ is closed in $\CG$.
		
		%\item (exhaustive) $E'$ is exhaustive since $E$ and $E'$ contain the same arguments of the form $(\{b\}\vdash b)$. Hence, we have $\prem(E)=\prem(E')$ and can conclude that $E'$ is exhaustive. 
		
		\item $E$ and $E'$ derive the same contraries:
		It holds that $y\in E$ since $S''\cup \{a\}\subseteq \asms(E)$, and since each complete set is assumption exhaustive. Hence, we obtain $y\in E$.
		\item $(E')^-_{\CG}\subseteq E^-_{\CF_D}$:
		In case $x\notin E$, we remove a potential attacker from our BAF. In case $x\in E$, we remove a potential target.
		\item (defense) $E'$ defends itself: 
		
		(Case 1: $x\notin E$) In this case, $E=E'$.
		As shown above, $E^-_{\CG}\subseteq E^-_{\CF_D}$. Since $E$ derives the same contraries in $\CF_D$ and $\CG$, we obtain that $E$ defends itself in $\CG$.
		
		(Case 2: $x\in E$) $E'$ attacks the same arguments as $E$ (as shown above, $E$ and $E'$ derive the same contraries).
        \item (fixed point) $E'$ is a fixed point of the characteristic function since $E'$ derives the same contraries as $E$. Hence each argument which is defended by $E$ is also defended by $E'$. 
	\end{itemize}
    This yields the proof for complete-based semantics.
	%
	%($\sigma\in\{\com,\grd,\stb\}$) 
	%For complete and grounded semantics, it remains to show that $E'$ contains all arguments it defends:
	%\begin{itemize}
	%	\item (fixed point) $E'$ is a fixed point of the characteristic function since $E'$ derives the same contraries as $E$. Hence each argument which is defended by $E$ is also defended by $E'$. 
	%\end{itemize}
	For stable semantics, we additionally observe that $E$ and $E'$ attack the same arguments $z\neq x$. Therefore, we obtain $E\in \sigma(\CG)$.
	
	%($\sigma\in\{\adm,\prf\}$) For admissible semantics, the result is immediate. For preferred semantics, we additionally observe that $\subseteq$-maximality is preserved.
	
	($\supseteq$)
	Take some $E\in\sigma(\CG)$. In this case, we have $x\notin E$. 
	
	(Case 1: $S\nsubseteq\asms(E)$) We show that $E\in\sigma(\CF_D)$. 
	\begin{itemize}
		\item 
		(conflict-free)
		It is clear that $E$ is conflict-free in $\CF_D$. 
		
		\item 
		(defense) 
		$E$ derives the same contraries in $\CG$ and $\CF_D$, hence, $E$ defends itself against all arguments in $A\setminus \{x\}$. 
		Assume $E$ is attacked by $\cl(x)$ in $\CF_D$.
		Then it is also attacked by $\cl(y)$ in $\CF_D$ (since $\cl(y)\subseteq \cl(x)$). Hence, $E$ is attacked by $\cl(y)$ in $\CG$. It thus defends itself against $\cl(y)$ due to being admissible. That is, $E$ contains an argument $(T\vdash q)$ with $q\in \contrary{\cl(y)}$. 
		Hence, $E$ attacks $\cl(x)$.
		Consequently, it defends itself against $\cl(x)$. 
		
		\item 
		(closed) $E$ it is closed in $\CG$, therefore, it is closed in $\CF_D$. 
		
		%\item (exhaustive) $E$ is exhaustive in $\CG$ hence it is exhaustive in $\CF_D$.
        \item (fixed point) 
		Since $S\nsubseteq\asms(E)$, there is at least one assumption $b\in S$ not defended by $E$. 
		Hence $E$ does not defend $x$ in $\CF_D$. 
	\end{itemize}
	
	%($\sigma\in\{\adm,\prf\}$) 
	%For admissible semantics, the result is immediate. For preferred semantics, we additionally observe that $E$ is $\subseteq$-maximal in $\CF_D$. 

	%($\sigma\in\{\com,\grd,\stb\}$) 
	%We show that $E$ is a fixed point of the characteristic function. 
	%\begin{itemize}
	%	\item (fixed point) 
	%	Since $S\nsubseteq\asms(E)$, there is at least one assumption $b\in S$ not defended by $E$. 
	%	Hence $E$ does not defend $x$ in $\CF_D$. 
	%\end{itemize}
	We obtain $E\in \sigma(\CF_\D)$ for $\sigma\in\{\com,\grd\}$.

	For stable semantics, we additionally observe that $E$ attacks $x$ in $\CF_D$. 
	There is at least one assumption $b$ which is not contained in $\asms(E)$. Since $E$ is stable in $\CG$ it holds that $E$ derives the contrary of $b$. Hence, $x$ is attacked in $\CF_D$.
	We obtain $E\in\stb(\CF_D)$.
	
	(Case 2: $S\subseteq \asms(E)$). We show that $E'=E\cup \{x\}\in\sigma(\CF_D)$. 
	\begin{itemize}
		\item 
		(conflict-free)
		$E'$ is conflict-free in $\CF_D$, otherwise, $E$ would be conflicting in $\CG$.
		
		\item 
		(defense)
		Each argument in $E$ is defended. Moreover, since $S\subseteq \asms(E)$, we obtain that $x$ is defended by $E'$.
		
		\item 
		(closed)
		This is clear. 
		%\item 
		%(exhaustive) We note that $\prem(x)\subseteq \prem(E')$, hence $E'$ is exhaustive in $\CF_D$.
		
	%\end{itemize}
	%($\sigma\in\{\adm,\prf\}$) 
	%For admissible semantics, the result is immediate. For preferred semantics, we additionally observe that $E$ is $\subseteq$-maximal in $\CF_D$. 

	%($\sigma\in\{\com,\grd,\stb\}$) 
	%We show that $E$ is a fixed point of the characteristic function. 
	%\begin{itemize}
		\item 
		(fixed point) $E'$ derives the same contraries as $E$ (analogous to above, we can deduce that $y\in E$). Hence each argument which is defended by $E$ is also defended by $E'$.
	\end{itemize}
	We obtain $E\in \sigma(\CF_\D)$ for all considered semantics. 
\end{proof}

\propNumberOfCoreArguments*
\begin{proof}
    Follows since each core argument is of the form $(S,a)$, $S\subseteq \mathcal{A}$, $a\in \mathcal{A}$.
\end{proof}

\subsection*{Redundancy-Free Core}
In the main paper, we assume that the redundancy-free core $\CG$ of an ABAF $D$ is well-defined. 
Here we formally show this. 
Let us recall how $A^*_D$ is constructed: 
From $A_D$ we construct the set $A^*_D$ of \emph{non-redundant} arguments by
i) first removing all derivation redundant arguments from $A_D$; 
ii) then removing all expendable arguments from the result of i); and finally 
iii) removing all assumption redundant arguments from the result of ii). 

The crucial part is to argue why these steps i) - iii) always output the same arguments. 

\begin{lemma}
    For any ABAF $D$, 
    the construction of the non-redundant arguments $A^*_D$ is well-defined. 
\end{lemma}
\begin{proof}
We prove correctness in each step.
\begin{itemize} 
    \item[i)] We show that derivation redundant arguments can be removed iteratively, in an arbitrary order.
    
    (no less arguments can be removed in step i) )
To this end suppose $x = S\vdash p$ is derivation redundant in $D$. We show that after removing an arbitrary set $X$ of derivation redundant arguments, $x$ is still derivation redundant (and thus removal of $x$ does not depend on the order in which we operate). 
Indeed, suppose $y$ is the witness for derivation redundancy of $x$, \ie $y_0 = S_0\vdash p$ with $S_0\subsetneq S$. 
If $y_0 \in X$, then there is a witness for derivation redundancy of $y_0$, \ie there is some $y_1 = S_1\vdash p$ with $S_1\subsetneq S_0 \subsetneq S$. 
By finitness of $\mathcal A$, this yields a finite sequence $y_n$ of arguments s.t.\ $y_n$ has not yet been removed. 
Now $y_n$ serves as witness for derivation redundancy of~$x$. 

(no more arguments can be removed in step i) )
Now suppose $x$ is not derivation redundant. It is trivial that it cannot become derivation redundant after an a set $X$ of arguments is removed. 

Moreover, removing derivation redundant arguments does not yield new expendable or assumption redundant arguments. 

    \item[ii)]
    For this redundancy notions, there is no interaction with other arguments. Hence the order in which we remove arguments does not matter.
    Moreover, removing expendable arguments does not yield new derivation redundant or assumption redundant arguments.

    \item[iii)] We show that assumption redundant arguments can be removed iteratively, in an arbitrary order.

    (no less arguments can be removed in step iii) )
    Suppose $x = S\vdash p$ is assumption redundant after step ii). We show that after removing an assumption redundant argument $y\neq x$, $x$ is still derivation redundant.
    Let $x',y'$ denote sub-arguments of $x,y$ witnessing assumption redundancy, respectively.

    In case $y\neq x'$ we are done. Then $x'$ witnesses assumption redundancy of $x$ after removing $y$.
    
    Now suppose $y=x'$. Then $y'$ is a sub-argument of $x$ with conclusion $a\in \mathcal{A}$ and leaves labelled with assumptions $S'\subseteq S$; hence, $y'$ witnesses assumption redundancy of $x$ after removing $y$. 

(no more arguments can be removed in step ii) )
Now suppose $x$ is not assumption redundant. The argument $x$ cannot become derivation redundant after removing another assumption redundant argument $y$.

Moreover, removing assumption redundant arguments does not yield new derivation redundant or expendable arguments. \qedhere 
\end{itemize}
\end{proof}
Consequently, $A^*_D$ is uniquely determined and hence the core 
induced by the set
$$\args = \{ (S, p) \mid S\vdash p \textnormal{ is an argument in } A^*_D \}$$ 
of arguments is well-defined.

\section{Omitted Proofs of Section~\ref{sec:fragments}}

\propComplexityPolyArgs*
\begin{proof}
	i) 
	Suppose the size of the core $A$ is polynomial. 
	
	(attack relation) 
	In order to construct the attack relation, for each pair $(A,p)$ and $(A',p')$ of core arguments, it has to be decided whether $p\in \contrary{A'}$. Thus the attack relation can be computed by iterating over each pair of core arguments. 
	
	(support relation) 
	In order to construct the support relation, for each argument $(A,p)$ and each assumption argument $\{a\}\vdash a$ it has to be decided whether $a\in\cl(A)$ holds. 
	By definition, $\cl(A) = \theory_D(A)\cap \mathcal A$; thereby, $\theory_D(A)$ can be computed (in polynomial time) by a unit propagation over the rule set $\mathcal R$. 
	Given $\theory_D(A)$, the intersection with $\mathcal A$ is easy to compute. 
	
	%(premises) 
	%If we seek to construct a pBAF, then $\prem ( (A,p) ) = A$ holds for each core argument, which is again trivial.
	
	All of theses steps are polynomial in $|A|$ and thus the claim follows. 
	
	ii) 
	The complexity follows now since we can construct the instantiated BAF in polynomial time and then perform the reasoning tasks there. 
	Thus the complexity is bounded by BAF complexity \cite[Theorem~6.1]{DBLP:journals/corr/abs-2305-12453}.
\end{proof}

%\propAdditiveABAF* 
\begin{restatable}{proposition}{propAdditiveABAF}
	\label{prop:additive abaf}
	Let $D$ be an additive ABAF. Then for each set $S$ of assumptions, 
	$\theory_D(S) = \bigcup_{s\in S}\theory_D(s)$. 
\end{restatable}
\begin{proof}
	$(\supseteq)$ 
	This holds by monotonicity of $\theory_D(\cdot)$. 
	
	$(\subseteq)$
	Striving for a contradiction, 
	suppose $S = s_1,\ldots,s_n$ exist s.t.\ 
	$\theory_D(S) \supsetneq \bigcup_{i=1}^n \theory_D(s_i)$. 
	Without loss of generality, assume $S$ is minimal with this property. 
	Assume $p\in \theory_D(S)$, but $p\notin \bigcup_{i=1}^n \theory_D(s_i)$ (such $p$ exists by assumption).  
	By minimality of $S$, $p$ cannot be derived from any of the sets $S\setminus \{s_1\},\ldots, S\setminus \{s_n\}$. 
	There is thus a tree-based argument $S\vdash p$, \ie the leaf nodes are labelled with $\top$ or some $s_i\in S$, and their union is the entire set $S$. 
	By the structure of trees, at least one derivation is involved where a rule $r$ is applied with $body(r)\geq 1$. 
	This is a contradiction to $D$ being additive. 
\end{proof}

\propArgNumberAtomic*
\begin{proof}
    Each assumption and each rule induces exactly one tree-based argument, giving rise to the upper bound.
\end{proof}
Note that the upper bound is not strict, e.g., a rule $(a gets a.)$ yields an argument $(\{a\}, a)$, corresponding to the assumption-argument for $a$ in the core. 

\propArgNumberAdditive*
\begin{proof}
	In general, the core consists of at most 
	%$|2^\mathcal A|\cdot |\mathcal L| + |\mathcal A|$ 
	$|2^\mathcal A|\cdot |\mathcal L|$ 
	many arguments. 
	However, for additive ABAFs, each tree derivation is of the form $\{\}\vdash p$ or $\{a\}\vdash p$ for some $a\in\mathcal A$ and $p\in\mathcal L$. 
	Consequently the term reduces to 
	%$(|\mathcal A|+1 )\cdot |\mathcal L| + |\mathcal A|$, 
	$(|\mathcal A|+1 )\cdot |\mathcal L|$, 
	which is quadratic in $|D|$. 
\end{proof}

\section{Algorithm details}
In this section, we provide the formal background to answer set programming (ASP)~\cite{GelfondL88,Niemela99} and Boolean satisfiability SAT~\cite{DBLP:series/faia/336}.
Further, we extend the algorithmic approaches we introduced in this paper to admissible semantics.

\subsection{Background on ASP and SAT}
An ASP program $\pi$ consists of rules $r$ of the form
$b_0 \la b_1,\ldots,b_k,\naf b_{k+1},\ldots,\ \naf b_m$, where each $b_i$ is an atom.
A rule is positive if $k=m$ and a fact if $m=0$. A literal is an atom $b_i$ or $\naf b_i$.
A rule without head $b_0$ is a constraint and a shorthand for $a \la b_1,\ldots,b_k,\naf b_{k+1},\ldots,\ \naf b_m, \naf a$ for a fresh $a$.
An atom $b_i$ is $p(t_1,\ldots,t_n)$ with each $t_j$ either a constant or a variable.
An answer set program
is ground if it is free of variables.
For a non-ground program,  $GP$ is the set of rules obtained by applying all possible
substitutions from the variables to the set of constants appearing in the program.
An interpretation $I$, i.e., a subset of all the ground atoms, satisfies a positive rule $r = h \la b_1,\ldots,b_k$
iff all positive body elements $b_1,\ldots,b_k$ being in $I$ implies that the head atom is in $I$.
For a program $\pi$ consisting only of positive rules, let $Cl(\pi)$ be the uniquely determined interpretation $I$
that satisfies all rules in $\pi$ and no subset of $I$ satisfies all rules in $\pi$.
Interpretation $I$ is an answer set of a ground program $\pi$ if $I = Cl(\pi^I)$
where  $\pi^I = \{(h \la b_1,\ldots,b_k) \mid (h \la b_1,\ldots,b_k,\naf b_{k+1},\ldots, \naf b_m)
\in \pi, \{b_{k+1},\ldots,b_m\} \cap I = \emptyset\}$ is the reduct; and
of a non-ground program $\pi$ if $I$ is an answer set of $GP$ of $\pi$.

The propositional satisfiability problem (SAT)~\cite{DBLP:series/faia/336} is the problem of deciding whether a given formula in propositional logic is satisfiable.
A formula is satisfiable if there is a truth assignment to the variables of the formula so that the formula evaluates to true.
%Due to the availability of efficient SAT solvers~\cite{DBLP:series/faia/0001LM21}, SAT-based algorithms are widely used for solving NP-hard problems.
%We will give the necessary background on the input format of SAT (namely propositional logic) and then consider the basics of modern SAT solvers.
SAT solvers restrict instances to be in conjunctive normal form (CNF), but any formula in propositional logic can be transformed into a CNF with linear overhead. %in a standard way~\cite{Tseitin1983}.
A CNF formula is a conjunction of clauses, which in turn comprise of a disjunction of literals.
A literal $l$ is either a variable $x$ or the negation of a variable $\neg x$.
A truth assignment sets each variable to either true or false.
A truth assignment satisfies a clause if it satisfies at least one literal in the clause.
The whole formula is satisfied by a truth assignment if each clause in the formula is satisfied.

\paragraph{CEGAR algorithms}
In a counterexample-guided abstraction-refinement (CEGAR)~\cite{DBLP:journals/jacm/ClarkeGJLV03,DBLP:journals/tcad/ClarkeGS04} algorithm, an over-approximation (abstraction) of the solution space is iteratively refined by drawing candidates from this space and verifying if the candidate is a solution.
Candidate solutions are computed with an NP-solver and another solver call is made to check if there is a counterexample to the candidate being a solution.
If there is no counterexample, the candidate is a solution to the original problem.
Otherwise the abstraction is refined by analyzing the counterexample, and the search is continued.

\subsection{ASP Encodings for Non-flat ABA}

We present the ASP encodings for our CEGAR algorithms for admissible and complete semantics, and the ASP encoding for stable semantics.
Listing~\ref{asp:adm-abs} shows $\aspmodule{adm\_abs}$, the abstraction used for candidate generation for admissible semantics.
The last rule of $\aspmodule{adm\_abs}$ is a slight pruning, prohibiting assumption sets where the undefeated assumptions are closed and attack the candidate assumption set.
The abstraction for complete semantics is 
$$\aspmodule{com\_abs} := \aspmodule{adm\_abs} \cup \{\la {\bf out}(X), \naf {\bf attacked\_by\_undefeated}(X)\}$$ 
(in Algorithm~\ref{alg:com-cred}, $\aspmodule{com\_abs}$ is called simply $\aspmodule{abs}$).
For credulous reasoning, the rule $\{\la \naf {\bf derivable}(s)\}$ is added to either abstraction.
In Listing~\ref{asp:verify1} we present the program $\aspmodule{not\_adm}$ for verifying the admissibility of a given candidate and in Listing~\ref{asp:verify2} the program $\aspmodule{defends}$ for verifying whether a given target assumption is defended by a given candidate.
The candidate assumptions are specified via the predicate $\aspin$ for both $\aspmodule{not\_adm}$ and $\aspmodule{defends}$, and the target assumptions for the latter is specified via the predicate ${\bf target}$.
Finally, Listing~\ref{asp:stable} presents $\aspmodule{stable}$, our encoding for finding a stable assumption set.
Compared to the ASP encoding for flat ABA~\cite{LehtonenWJ21a}, in $\aspmodule{stable}$ we enforce that a stable assumption set must be closed via the constraint on Line 7, excluding a case where an assumption is derivable from the in set, but not contained in the in set.

\begin{listing}[t]
  \caption{Module $\aspmodule{adm\_abs}$\label{asp:adm-abs}}
%  \begin{lstlisting}[caption={Module $\aspmodule{common}$},frame=lines,label=asp:common,float=tp,floatplacement=tbp]
\begin{lstlisting}
in(X) :- assumption(X), not out(X).
out(X) :- assumption(X), not in(X).
derivable(X) :- assumption(X), in(X).
derivable(X) :- head(R,X), usable_by_in(R).
usable_by_in(R) :- head(R,_), derivable(X) : body(R,X).
:- in(X), contrary(X,Y), derivable(Y).
:- assumption(X), derivable(X), not in(X).
defeated(X) :- derivable(Y), contrary(X,Y).
derivable_from_undefeated(X) :- assumption(X), not defeated(X).
derivable_from_undefeated(X) :- head(R,X), usable_by_undefeated(R).
usable_by_undefeated(R) :- head(R,_), derivable_from_undefeated(X) : body(R,X).
att_by_undefeated(X) :- contrary(X,Y), derivable_from_undefeated(Y).
no_undef_closed :- assumption(X), derivable_from_undefeated(X), not defeated(X).
:- in(X), att_by_undefeated(X), not no_undef_closed.
\end{lstlisting}
\end{listing}

\begin{listing}[t]
  \caption{Module $\aspmodule{not\_adm}$\label{asp:verify1}}
    \begin{lstlisting}
{guess(X)} :- not defeated(X), assumption(X).
in_derives(X) :- assumption(X), in(X).
in_derives(X) :- head(R,X), usable_by_in(R).
usable_by_in(R) :- head(R,_), in_derives(X) : &\newline& body(R,X).
guess_derives(X) :- assumption(X), guess(X).
guess_derives(X) :- head(R,X), usable_by_guess(R).
usable_by_guess(R) :- head(R,_), guess_derives(X) : body(R,X).
:- assumption(X), guess_derives(X), not guess(X).
:- guess(X), contrary(X,Y), in_derives(Y).
in_defeated_by_guess :- in(X), contrary(X,Y), guess_derives(Y).
:- not in_defeated_by_guess.
\end{lstlisting}
\end{listing}

\begin{listing}[t]
  \caption{Module $\aspmodule{defends}$\label{asp:verify2}}
    \begin{lstlisting}
{guess(X)} :- not defeated(X), assumption(X).
in_derives(X) :- assumption(X), in(X).
in_derives(X) :- head(R,X), usable_by_in(R).
usable_by_in(R):-head(R,_), in_derives(X) : body(R,X).
guess_derives(X) :- assumption(X), guess(X).
guess_derives(X) :- head(R,X), usable_by_guess(R).
usable_by_guess(R) :- head(R,_), guess_derives(X) : body(R,X).
:- assumption(X), guess_derives(X), not guess(X).
:- guess(X), contrary(X,Y), in_derives(Y).
guess_defeats(X) :- contrary(X,Y), guess_derives(Y).
:- target(X), not guess_defeats(X).
\end{lstlisting}
\end{listing}

\begin{listing}[t]
  \caption{Module $\aspmodule{stable}$\label{asp:stable}}
    \begin{lstlisting}
in(X) :- assumption(X), not out(X).
out(X) :- assumption(X), not in(X).
derivable(X) :- assumption(X), in(X).
derivable(X) :- head(R,X), usable_by_in(R).
usable_by_in(R) :- head(R,_), derivable(X) : body(R,X).
:- in(X), contrary(X,Y), derivable(Y).
:- assumption(X), derivable(X), not in(X).
defeated(X) :- derivable(Y), contrary(X,Y).
:- out(X), not defeated(X).
\end{lstlisting}
\end{listing}

\section{Admissible Semantics and pBAFs}

Due to space restrictions, the main part of this paper focuses on complete-based semantics. Here we sketch similar findings for admissible semantics. 

\subsection{Bipolar argumentation and premises}
We recall the definitions following~\cite{DBLP:journals/corr/abs-2305-12453}.
\begin{definition}
	A \emph{premise-augmented bipolar argumentation framework (pBAF)} $\PF$ is a tuple $\PF = (\args, \attacker, \supporter, \prem)$ 
	where 
	$\args$ represents a set of arguments,
	$\attacker\subseteq \args\times \args$ models \textit{attacks} and $\supporter \subseteq \args \times \args$ models \emph{support} between them,
	and $\prem:\args\to 2^\allprem$ is the \emph{premise function}, $\allprem$ is a set (of premises).    
\end{definition}
We let $\prem(E) = \bigcup_{a\in E} \prem(a)$. 
For a set of arguments $E\subseteq A$, we let $E^+_{\CF} = \{ x\in A \mid E\text{ attacks }x \}$.
We drop the subscript whenever it does not cause any confusion.
We call the pair $F = (\args,\attacker)$ the {underlying} AF and $\CF = (\args, \attacker, \supporter)$ is the {underlying} BAF of $\PF$.
We sometimes abuse notation and write 
$\PF = (\CF, \prem)$ 
for the pBAF 
$\PF = (\args, \attacker, \supporter, \prem)$ 
with underlying BAF 
$\CF= (\args, \attacker, \supporter)$. 

We utilize the (p)BAF semantics from \cite{DBLP:journals/corr/abs-2305-12453}.
The following definitions do not involve premises and are therefore  analogous to the BAF definitions. 
\begin{definition}
	Let $\PF = (\CF,\prem)$ be a pBAF and let $E\subseteq A$.
	The set $E$ is \emph{closed} iff $E$ is closed in $\CF$; 
	$E$ is \emph{conflict-free} in $\PF$, iff $E$ is closed in $\CF$; 
	$E$ \emph{defends} $a\in \args$ iff $E$ defends $a$ in $\CF$.
	The \emph{characteristic function} of $\PF$ is $\Gamma(E) = \{ a\in A\mid E\text{ defends }a \}$. 
\end{definition}
The complete-based semantics for pBAFs and BAFs coincide, so we do not repeat them. 
%Observe that we do not make use of the premises in the following definition.
%\begin{definition} 
%	Let $\PF$ be a BAF. For $E$ closed and $E\in\cf(\PF)$, 
%	\begin{itemize}
%		\item $E\in\com(\PF)$ iff $E = \Gamma(E)$; 
%		\item $E\in\grd(\PF)$ iff $E = \bigcap_{S\in\com(\PF)} S$; 
%		\item $E\in\oldprf(\PF)$ iff $E$ is $\subseteq$-maximal in $\adm(\PF)$; 
%		\item $E\in\stb(\PF)$ iff $E^+=A\setminus E$.
%	\end{itemize}
%\end{definition}
Thus, we define admissible-based semantics for pBAFs. 
For this, we require the notion of \emph{exhaustiveness}.
\begin{definition}
	Let $\PF = (\CF,\prem)$ be a pBAF. A set $E\subseteq A$ is \emph{exhaustive} iff $\prem(a)\subseteq \prem(E)$ implies $a\in E$.
\end{definition}
Now we define admissibility as in the case of BAFs, but we additionally require exhaustiveness of the set. 
The rationale for this is given in~\cite{DBLP:journals/corr/abs-2305-12453}. 
\begin{definition}
	Let $\PF$ be a pBAF. Then 
    \begin{itemize}
        \item $E\in\adm(\PF)$ iff $E\in\adm(\CF)$ and $E$ is exhaustive;
        \item $E\in\oldprf(\PF)$ iff $E$ is $\subseteq$-maximal in $\adm(\PF)$.
    \end{itemize}
\end{definition}
\paragraph{ABA and pBAFs}
To capture ABAFs by means of pBAFs w.r.t.\ admissible semantics we employ the same instantiation. 
The only difference is that for an argument $S\vdash p$, we additionally keep track of the set $S$ as underlying \emph{premise}, \ie $\prem( S\vdash p ) = S$. 
\begin{definition}
	For an ABAF  
	$D = (\mathcal{L},\mathcal{R},\mathcal{A},\contraryempty)$, 
	the %corresponding 
	\emph{{instantiated} pBAF} 
	$\PF_D = (\args,\attacker,\supporter,\prem)$ 
	is 
	\begin{align*}
		A &= \{ (S\vdash p) \mid (S\vdash p)\text{ is an argument in }D \}\\
		\attacker  &= \{ ( S\vdash p, T\vdash q ) \in A^2 \mid p \in\contrary{T} \}\\
		\supporter &= \{ ( S\vdash p, \{a\} \vdash a ) \in A^2 \mid a \in\cl(S) \} 
	\end{align*} 
	and $\prem(x) = \asms(x)$ for all $x \in \args$.
	$\CF_D=(\args,\attacker,\supporter)$.
\end{definition}
As shown in \cite{DBLP:journals/corr/abs-2305-12453}, we now get the semantics correspondence for $\adm$ as well. 
\begin{theorem}
	\label{thm:SemanticsCorrespondenceADM}
	\label{th:semantics correspondence 2}
	Let $D = (\mathcal{L},\mathcal{R},\mathcal{A},\contraryempty)$ be an ABAF 
	and $\PF_D = (\CF_D,\prem)$ the %corresponding 
	instantiated pBAF. Then  
	\begin{itemize}
		\item if $E\in\sigma(\PF_D)$, then $\asms(E)\in\sigma(D)$; 
		\item if $S\in\sigma(D)$, then $\{ x\in A \mid \asms(x)\subseteq S \}\in\sigma(\PF_D)$ 
	\end{itemize}
	for any $\sigma\in\{\adm,\com,\oldprf,\prf,\grd,\stb\}$. 
\end{theorem}

\subsection{Redundancy Notions for pBAFs}

All redundancy notions we established also hold for pBAFs. 
We give the ``pBAF'' version of those. 
Note that we do not give the full proofs again, but we follow the structure and merely add remarks regarding eshaustiveness. 

\begin{proposition}
	\label{prop:redundant args pbaf}
	Let $D$ be an ABAF, $\PF_D$ the corresponding pBAF, and \linebreak $\sigma\in\{\adm,\com,\oldprf,\prf,\grd,\stb\}$.
	Let $x$ be derivation redundant and let $\PG$ be the pBAF after removing the argument $x$ from $\PF_D$. 
	Then 
	$$ \{ \asms(E)  \mid E\in\sigma(\PF_D) \} = \{ \asms(E)  \mid E\in\sigma(\PG) \}. $$
\end{proposition}
\begin{proof}
%\TODO{proof does not really argue for exhaustive}
	Let $x = (S\vdash p)$ and $y = (S'\vdash p)$ with $S'\subsetneq S$; such $y$ must exist because $x$ is redundant. 
	
	($\subseteq$)
	Take some $E\in\sigma(\PF_D)$. We construct a corresponding extension in $\PG$. 
	
	(Case 1: $x\notin E$) 
	In this case, $E\in\sigma(\PG)$ 
	
	(Case 2: $x\in E$)
        We have already shown the result for complete-based semantics, so we have left to consider $\adm$ and $\oldprf$. 
	\begin{itemize}
		\item 
		($\sigma = \adm$) 
		Due to $\asms(y)\subseteq \asms(x)$, $E$ defends the set $\{ (\{a\} \vdash a) \mid a\in\asms(x) \}$ of assumption arguments and $E$ defends $y$. 
		Set $$E' = E\cup \{y\} \cup \{ (\{a\} \vdash a) \mid a\in\asms(x) \}.$$ 
        Due to 
		$\asms(E) = \asms(E')$ we have $E'$ is exhaustive; moreover, 
		$\cl(E) = \cl(E')$ by Lemma~\ref{le:additive closure};
		hence $E'$ is closed.
        We reason similar as in the proof of Proposition~\ref{prop:redundant args} to see that $E'\setminus \{x\}$ is admissible in $\CG$. 
        \item ($\sigma=\oldprf$) This is a direct consequence from the correspondence of admissible semantics. 
	\end{itemize}
	
	($\supseteq$)
	Take some $E\in\sigma(\CG)$. In this case, we have of course $x\notin E$. 
	
	(Case 1: $S\nsubseteq\asms(E)$) We show that $E\in\sigma(\PF_D)$. 
	
	\begin{itemize}
		\item 
		(conflict-free)
		It is clear that $E$ is conflict-free in $\PF_D$. 
		
		\item 
		(defense) 
		If $E$ is attacked by $\cl(\{x\})$ in $\PF_D$, then it is also attacked by $\cl(\{y\})$ in $\PF_D$. Hence it is attacked by $\cl(\{y\})$ in $\PG$. It thus defends itself against $y$ due to being admissible. Consequently, it defends itself against $\cl(\{x\})$. 
		Defense against any other argument is clear. 
		
		\item 
		(closed \& exhaustive)
        This is clear.
	\end{itemize}
	
	(Case 2: $S\subseteq\asms(E)$). By similar reasoning, $E\cup \{x\}\in\sigma(\PF_D)$. 
\end{proof}
\begin{proposition}
	\label{prop:expendable args pbaf}
	Let $D$ be an ABAF, $\PF_D$ the corresponding pBAF, and \linebreak $\sigma\in\{\adm,\com,\oldprf,\prf,\grd,\stb\}$.
	Let $x$ be an expendable argument and let $\PG$ be the pBAF after removing the argument $x$ from $\PF_D$. 
	Then 
	$$ \{ \asms(E)  \mid E\in\sigma(\PF_D) \} = \{ \asms(E)  \mid E\in\sigma(\PG) \}. $$
\end{proposition}
\begin{proof}[Sketch]
	We reason as in the proof of Proposition~\ref{prop:redundant args}. 
	Let $x = (S\vdash p)$. 
	Let $E\in\adm(\CF)$. 
	
	By admissibility, all arguments in $\{ (\{a\} \vdash a ) \mid a \in S \}$ are defended by $E$. 
	Hence the proof of Proposition~\ref{prop:redundant args pbaf} translates to this situation. 
\end{proof}
\begin{proposition}
	\label{prop:assumption redundant args pbaf}
	Let $D$ be an ABAF, $\PF_D$ the corresponding pBAF, and \linebreak $\sigma\in\{\adm,\com,\oldprf,\prf,\grd,\stb\}$.
	%Let $\sigma\in \{ \com,\grd,\stb ,\adm,\prf\}$.
	Let $x$ be an assumption redundant argument and let $\CG$ be the pBAF after removing the argument $x$ from $\PF_D$. 
	Then 
	$$ \{ \asms(E)  \mid E\in\sigma(\PF_D) \} = \{ \asms(E)  \mid E\in\sigma(\PG) \}. $$
\end{proposition}
\begin{proof}
	Let $x = (S\vdash p)\in A$ and let $x'=(S'\vdash a)$ be the proper sub-argument of $x$ with $S'\subseteq S$ and $a\in\mathcal{A}$, witnessing its assumption redundancy. Furthermore, let $y$ denote the argument which arises when removing $x'$ from $x$, i.e., $y=(S''\cup \{a\}\vdash p)$ with $S''\subseteq S$.
	
	($\subseteq$)
	Let $E\in\sigma(\PF_D)$. We show that $E'=E\setminus \{x\}\in\sigma(\PG)$. 
        Due to our BAF results, we have left to show the proof for exhaustiveness. 
	
	\begin{itemize}
		\item (exhaustive) $E'$ is exhaustive since $E$ and $E'$ contain the same arguments of the form $(\{b\}\vdash b)$. Hence, we have $\prem(E)=\prem(E')$ and can conclude that $E'$ is exhaustive. 
	\end{itemize}
	
	($\supseteq$)
	Take some $E\in\sigma(\PG)$. In this case, we have $x\notin E$. 
	
	(Case 1: $S\nsubseteq\asms(E)$) To show $E\in\sigma(\PF_D)$, we observe that 
 $E$ is exhaustive in $\CG$ hence it is exhaustive in $\PF_D$.
 Together with Proposition~\ref{prop:assumption redundant args} we obtain the results for all considered pBAF semantics. 
\end{proof}

\subsection{Algorithms for Admissible Semantics}
In this chapter we extend both our translation-based approach and the CEGAR approach for admissible semantics.

\paragraph{SAT Encodings for Admissible Semantics}

    Extending the SAT encodings proposed in Section~\ref{sec:baf-sat} to admissible semantics requires auxiliary variables for the inclusion of assumptions in extensions, i.e. for each $\alpha\in\mathcal{A}$, we have variable $x_{\alpha}$, which is true if and only if some argument in the extension corresponding to an interpretation contains $\alpha$.
    Adding this condition to the clauses for conflict-freeness, closedness and self-defense, we get the encoding for admissible semantics.
    \begin{align*}
        adm(F) =&cf(F)\wedge closed(F)\wedge self\_defense(F)\\ &\wedge \bigwedge_{a\in A} \left( x_a \leftrightarrow \bigwedge_{\alpha\in \prem(a)} x_{\alpha} \right)
    \end{align*}

    % \TODO{Include skeptical at all? If so, move these somewhere. This uses the ordinary definition for skeptical, not the one we define in this paper}
    %     To decide whether atom $\alpha$ is skeptically accepted, we add the clause $$skept(F,\alpha):=\bigwedge_{\alpha\in \prem(a)} \neg x_a.$$
    % The answer is positive if $\sigma(F)\wedge skept(F,\alpha)$ is unsatisfiable.

\paragraph{ASP-based CEGAR Algorithm for Admissible Semantics}
    Our algorithm for credulous acceptance under admissible semantics is presented in Algorithm~\ref{alg:adm-cred}.
    The abstraction $\aspmodule{adm\_cred\_abs}$ is otherwise the same as the one used for Algorithm~\ref{alg:com-cred} but without the additional condition that there are no assumptions that are not in the candidate and not attacked by the set of undefeated assumptions.
    In Lines~\ref{alg:adm-cred}--\ref{alg:adm-cred-3}, Algorithm~\ref{alg:adm-cred} iteratively draws candidates from the abstraction (Line~\ref{alg:adm-cred-1}), checks whether there is a counterexample to the candidate being admissible (Line~\ref{alg:adm-cred-2}), and refines the abstraction to rule the candidate from further consideration if so (Line~\ref{alg:adm-cred-3}).
    If there are no counterexamples, the candidate is an admissible set and therefore the query is credulously accepted (Line~\ref{alg:adm-cred-2}).
    If no admissible assumption set is found, the algorithm finally returns NO (Line~\ref{alg:adm-cred-4}).
    
    \begin{algorithm}[t]
    \caption{Credulous acceptance, admissible semantics}
    \label{alg:adm-cred}
    \begin{algorithmic}[1]
        \REQUIRE ABA framework $F=(\mathcal{L},\mathcal{R},\mathcal{A},\contraryempty)$, $s\in\mathcal{L}$
        \ENSURE return YES if $s$ is credulously accepted under admissible semantics in $F$, NO otherwise
        \STATE{$\algorithmicwhile\ C:= solve(\aspmodule{adm\_cred\_abs})\ \algorithmicdo$} \label{alg:adm-cred-1}
        %\STATE{\hspace{\algorithmicindent}Let $C$ be a candidate}
        %\STATE{\hspace{\algorithmicindent}$\algorithmicif\ $ Candidate is admissible %$\algorithmicreturn$ YES} \label{alg:adm-cred-5}
        \STATE{\hspace{\algorithmicindent}$\algorithmicif\ solve(\aspmodule{not\_adm}(C))$ unsatisfiable$\ \algorithmicreturn$ YES} \label{alg:adm-cred-2}
        \STATE{\hspace{\algorithmicindent}Add constraint excluding $C$ to $\aspmodule{adm\_cred\_abs}$} \label{alg:adm-cred-3} 
        \RETURN{NO} \label{alg:adm-cred-4}
    \end{algorithmic}
    \end{algorithm}

    % \TODO{Include skeptical at all? If so, move these somewhere. This uses the ordinary definition for skeptical, not the one we define in this paper}
    % To modify either CEGAR algorithm to decide skeptical acceptance, it suffices to require in the abstraction that the query is \emph{not} derivable (instead of being derivable as in the abstractions for credulous acceptance) from the candidate, and invert the answer (i.e. return YES when Algorithm~\ref{alg:adm-cred}, or respectively Algorithm~\ref{alg:com-cred}, would report NO, and vice versa).
    % This is because for skeptical acceptance, if we find a $\sigma$-assumption set that does not contain the query, then the query is not accepted.

    \section{Additional Empirical Results}

Tables~\ref{table:generic-dc-ad} and~\ref{table:generic-dc-st} shows the performance of \ababaf{} and \cegar{} on credulous acceptance under admissible and stable semantics in benchmark set 1, respectively, and tables~\ref{table:patomic-dc-ad} and~\ref{table:patomic-dc-st} on set 2.
The situation for admissible is quite similar to complete, with \ababaf{} performing slightly better and \cegar{} worse.
In contrast, for stable semantics \cegar{} performs considerably better, as can be expected given that \cegar{} solves the task with a single ASP call, while \ababaf{} needs to perform the whole argument construction process, similarly to the other semantics.

\begin{table}[t]
    %\vspace{-2mm}
    \centering
    %\resizebox{\linewidth}{!} 
    {
\small
    \begin{tabular}{cc|rr|rr}
        \toprule
\multicolumn{2}{c}{} & \multicolumn{4}{c}{\textbf{\#solved} (mean run time (s))}\\
%\hline
\midrule
$ms$ &
$mr$ &
\multicolumn{2}{c|}{\ababaf{}} &
\multicolumn{2}{c}{\cegar{}} \\
%\hline
\midrule
  & 1 & 80 & (0.2) & 43 & (8.3) \\
1 & 2 & 80 & (0.4) & 60 & (25.2) \\
  & 5 & 80 & (3.7) & 80 & (0.1) \\
  \midrule
  & 1 & 79 & (0.5) & 42 & (17.3) \\
2 & 2 & 54 & (16.3) & 34 & (53.0) \\
  & 5 & 51 & (14.2) & 60 & (11.8) \\
  \midrule
  & 1 & 80 & (2.1) & 70 & (7.1) \\
5 & 2 & 54 & (3.9) & 60 & (0.5) \\
  & 5 & 11 & (23.6) & 37 & (45.0) \\
        \bottomrule 
    \end{tabular}
}
    \caption{Number of solved instances and mean run time over solved instances on credulous reasoning under admissible semantics in benchmark set 1. There are 80 instances per row.
     \label{table:generic-dc-ad}}
\end{table}

\begin{table}[t]
    %\vspace{-2mm}
    \centering
    %\resizebox{\linewidth}{!} 
    {
\small
    \begin{tabular}{cc|rr|rr}
        \toprule
\multicolumn{2}{c}{} & \multicolumn{4}{c}{\textbf{\#solved} (mean run time (s))}\\
%\hline
\midrule
$ms$ &
$mr$ &
\multicolumn{2}{c|}{\ababaf{}} &
\multicolumn{2}{c}{\cegar{}} \\
%\hline
\midrule
  & 1 & 80 & (0.2) & 80 & (0.1) \\
1 & 2 & 80 & (0.4) & 80 & (0.1) \\
  & 5 & 80 & (1.4) & 80 & (0.1) \\
  \midrule
  & 1 & 80 & (10.) & 80 & (0.1) \\
2 & 2 & 65 & (13.6) & 80 & (0.1) \\
  & 5 & 64 & (13.7) & 80 & (0.1) \\
  \midrule
  & 1 & 80 & (0.4) & 80 & (0.1) \\
5 & 2 & 59 & (3.6) & 80 & (0.1) \\
  & 5 & 20 & (67.5) & 80 & (0.1) \\
        \bottomrule 
    \end{tabular}
}
    \caption{Number of solved instances and mean run time over solved instances on credulous reasoning under stable semantics in benchmark set 1. There are 80 instances per row.
     \label{table:generic-dc-st}}
\end{table}

\begin{table}[t]
    %\vspace{-2mm}
    \centering
    %\resizebox{\linewidth}{!} 
    {
\small
    \begin{tabular}{cc|rr|rr}
        \toprule
\multicolumn{2}{c}{} & \multicolumn{4}{c}{\textbf{\#solved} (mean run time (s))}\\
%\hline
\midrule
$slack$ &
$ms$ &
\multicolumn{2}{c}{\ababaf{}} &
\multicolumn{2}{c}{\cegar{}} \\
\midrule
0 & 2 & 135 & (2.3) & 59 & (58.4) \\
  & 5 & 121 & (7.2) & 80 & (49.6) \\
  \midrule
1 & 2 & 94 & (13.6) & 72 & (31.7) \\
  & 5 & 69 & (17.5) & 96 & (18.3) \\
  \midrule
2 & 2 & 110 & (11.7) & 93 & (18.7) \\
  & 5 & 71 & (17.3) & 96 & (32.2) \\
  \bottomrule
\end{tabular}
}
    \caption{Number of solved instances and mean run time over solved instances on credulous reasoning under admissible semantics in benchmark set 2. There are 160 instances per row.
     \label{table:patomic-dc-ad}}
\end{table}

\begin{table}[t]
    %\vspace{-2mm}
    \centering
    %\resizebox{\linewidth}{!} 
    {
\small
    \begin{tabular}{cc|rr|rr}
        \toprule
\multicolumn{2}{c}{} & \multicolumn{4}{c}{\textbf{\#solved} (mean run time (s))}\\
%\hline
\midrule
$slack$ &
$ms$ &
\multicolumn{2}{c}{\ababaf{}} &
\multicolumn{2}{c}{\cegar{}} \\
\midrule
0 & 2 & 147 & (9.8) & 160 & (0.1) \\
  & 5 & 137 & (10.5) & 160 & (0.1) \\
  \midrule
1 & 2 & 118 & (14.3) & 160 & (0.1) \\
  & 5 & 91 & (26.9) & 160 & (0.1) \\
  \midrule
2 & 2 & 126 & (11.6) & 160 & (0.1) \\
  & 5 & 81 & (6.8)& 160 & (0.1) \\
  \bottomrule
\end{tabular}
}
    \caption{Number of solved instances and mean run time over solved instances on credulous reasoning under stable semantics in benchmark set 2. There are160 instances per row.
     \label{table:patomic-dc-st}}
\end{table}

\end{document}